\documentclass{article}

\usepackage{arxiv}

\usepackage[utf8]{inputenc} 
\usepackage[T1]{fontenc}    
\usepackage{hyperref}       
\usepackage{url}            
\usepackage{booktabs}       
\usepackage{amsfonts}       
\usepackage{nicefrac}       
\usepackage{microtype}      
\usepackage{lipsum}		
\usepackage{graphicx}
\usepackage{natbib}
\usepackage{doi}
\usepackage{setspace}
\usepackage{float}

\usepackage{amsthm,amsmath} 
\usepackage{mathrsfs} 
\usepackage{amsmath}
\usepackage{amssymb}
\usepackage{latexsym}
\usepackage{bm}
\usepackage{color}
\usepackage{url} 
\usepackage[title]{appendix}%
\usepackage{xcolor}%
\usepackage{textcomp}%
\newcommand{\choosefont}[1]{\fontfamily{#1}\selectfont}

\usepackage{mathtools} 
\usepackage{algorithm}
\usepackage{algorithmicx}%
\usepackage{algpseudocode}%
\usepackage{chngcntr} 
\usepackage[title]{appendix}
\newtheorem{thm}{Theorem}
\newtheorem*{thm*}{Theorem}

\theoremstyle{definition}

\newtheorem{lem}[thm]{Lemma} 

\def\argmin{\mathop{\rm argmin}}

\theoremstyle{thmstyletwo}%
\newtheorem{remark}{Remark}%

\theoremstyle{thmstylethree}%
\title{Geographically Regularized AUC-Maximizing Personalized Federated Learning}

\author{ {\hspace{1mm}Mayu Hiraishi} \\
	Department of Medicine\\
	Wakayama Medical University\\
	Wakayama, Japan \\
	\texttt{m-hira@wakayama-med.ac.jp} \\
	\And
	{\hspace{1mm}Kensuke Tanioka} \\
	Department of Biomedical Sciences and Informatics\\
	Doshisha University\\
	Kyoto, Japan\\
    \texttt{ktanioka@mail.doshisha.ac.jp}\\
    \And
    {\hspace{1mm}Toshio Shimokawa} \\
	Department of Medicine\\
	Wakayama Medical University\\
	Wakayama, Japan \\
    \texttt{shimokaw@wakayama-med.ac.jp}\\
}

\renewcommand{\shorttitle}{GrAUC-PFL}

\hypersetup{
}

\begin{document}
\maketitle

\begin{abstract}
Accurate diagnostic and risk-prediction models are important for supporting clinical decision-making during infectious disease outbreaks. However, privacy and governance requirements may restrict patient-level data sharing across healthcare institutions, and data distributions often vary. Moreover, AUC is widely used to evaluate discriminative performance, motivating its direct optimization in model development. We propose geographically regularized AUC-maximizing personalized federated learning (GrAUC-PFL), which directly optimizes a smooth pairwise AUC surrogate to learn personalized models while keeping patient-level data local and accounting for institutional heterogeneity. Graph-based regularization encourages geographically neighboring institutions to have similar coefficient vectors while retaining a personalized models. Simulations and a real-data application suggest improved discriminative performance, particularly when geographically neighboring institutions have similar data-generating characteristics.
\end{abstract}

\keywords{AUC maximization \and geographic proximity \and graph regularization \and medical diagnosis \and personalized federated learning}

\section{Introduction}
\label{sec:introduction}

During large-scale infectious disease epidemics, such as the COVID-19 pandemic, accurate diagnostic and risk prediction models are essential for identifying patients at high risk of severe disease, supporting timely clinical decision-making, and allocating limited healthcare resources efficiently. Developing reliable diagnostic models often benefits data collected across multiple healthcare institutions. However, in recent years, privacy concerns and regulatory requirements have often restricted patient-level data sharing, resulting in difficulties in aggregating data across multiple institutions \citep{info17020148}. For example, the General Data Protection Regulation in the European Union \citep{GDPR2016} imposes strict requirements on patient-level data sharing across institutions, making centralized data aggregation challenging in practice. 

In such situations, Federated learning (FL) \citep{FedAvg} enables collaborative model training without directly sharing patient-level data, which have been applied to various medical applications \citep[e.g.][]{FL_predictCOVID, FL_Melanoma}. However, clinical data often exhibit heterogeneity across institutions or regions owing to differences in patient populations, clinical practice, and evaluation criteria, which may not be adequately captured by a single global model used in conventional FL. Personalized federated learning (PFL) \citep{pfl_falla,  pfl_jiang} addresses this limitation by training personalized models, whilst utilizing information from other institutions.

In addition, healthcare institutions located in close geographical proximity are observed to share similar patient demographics and diagnostic patterns. 
In fact, hospitals serve geographically defined catchment areas, which may affect COVID-19 hospitalization forecasts \citep{Meakin2024}, while demographic structure and population movement can lead to similar epidemiological dynamics in neiboring regions \citep{Rader2020}. These observations suggest that geographic proximity can provide relevant information for one another in infectious disease settings. Therefore, beyond relationships inferred from the model parameters or observed data, geographic proximity can provide an additional information on similarities among institutions. 
This concept has been exploited in spatial statistical methods \citep[e.g.][]{GWR,SVP}. 
Therefore, in the PFL, incorporating geographical proximity is expected to improve the estimation accuracy of local models by utilising information from other participating organisations. 

Furthermore, class imbalances are likely to arise in outcomes. Consequently, in such situations, direct AUC maximization \citep{Yan_AUCmax_2003, AUCsurvey} can provide improved discriminative performance by explicitly optimizing scores based on the ranking of positive and negative cases \citep{pmlr-v80-natole18a}. Therefore, we employ AUC direct maximization. 

Motivated by these challenges, we propose geographically regularized AUC-maximizing personalized federated learning (GrAUC-PFL), which directly maximizing a smooth pairwise AUC surrogate within a PFL framework. GrAUC-PFL uses graph-based regularization that encourages geographically neighboring institutions to have similar model parameters. This allows each institution to retain personalized diagnostic models, while selectively borrowing information from its neighbors. GrAUC-PFL integrates direct AUC maximization with externally defined geographical structure within a unified PFL framework. Unlike similarity-based PFL approaches such as \cite{FedAMP, PFLG}, which infer client relationships from local model updates or parameter similarities, GrAUC-PFL uses geographical proximity as a prior structural information sharing independently to be guided by external knowledge rather than potentially noisy local model estimates. This may be particularly beneficial when local data are limited.

Because the proposed optimization problem combines a non-decomposable AUC objective function with graph-based regularization, we develop an efficient optimization algorithm based on the Alternating Direction Method of Multipliers (ADMM) \citep{boydadmm}, extending the optimization framework of \cite{perFL-RSR}. Thus, the proposed framework provides a computationally efficient approach for estimating personalized diagnostic models from geographically distributed medical data.

The remainder of this article is organised as follows. Section \ref{sec:RW} explain the related works of GrAUC-PFL. Section \ref{sec:proposed_method} introduces GrAUC-PFL, the objective function, and the algorithm. A numerical simulation is reported in Section \ref{sec:simulation} and a real data application is presented in Section \ref{sec:realdata}. We discuss the results of numerical simulation and real-data application in Section \ref{sec:discussion} and conclude the article with Section \ref{sec:conclusion}.

\section{Related works}
\label{sec:RW}

\subsection{Personalized federated learning (PFL)}
\label{subsec:PFL}

In this section, we briefly review conventional federated learning (FL) and then introduce PFL. 
FL is a distributed learning framework in which multiple institutions collaboratively train a single global model without sharing raw data. 
In standard FL, each institution sends model updates (e.g., parameters or gradients) to the server, enabling the global model to be updated without the need to provide data to external institutions. In the case of client-server structure, the server distributes the global model to each institution, and each institution updates the model using local data. The locally updated model parameters are transmitted to the central server, where they are aggregated (e.g. averaged), and then shared with each institution once again as the global model. A shared global model is estimated over multiple communication rounds. The image on the left in Figure \ref{fl_pfl} describes an example of the structure of FL. 

However, it may not be possible for the single global model to account adequately for the heterogeneity of institutions. For example, if the true models differ across institutions, the resulting global model may fail to achieve high AUC at individual institutions. Furthermore, patient demographics, sample sizes and prevalence rates often vary between institutions. When information is integrated, institutions with larger sample sizes may dominate the optimization, resulting in poorer performance at smaller institutions. To overcome this issue, the PFL framework has been proposed. 
In PFL, updating the local model using local data is the same procedure as in standard FL. When the local model or parameter are aggregated at the central server, each institution retains its own personalized model while leveraging information from other institutions, which is described at the right side of Figure \ref{fl_pfl}. In PFL, when estimating personalized models, information from other institutions without sharing raw data stored at each institution. Various approaches have been proposed to leverage information from other institutions. 

\cite{pfl_falla} extends FedAvg \citep{FedAvg} by incorporating gradient correction for each institution, thus taking personalization into account. Another approach adds the regularization term penalizing the deviation between the personalized and global models \citep{Li2020DittoFA}. In terms of group structure, clustered federated learning \cite{Sattler2019ClusteredFL} uses the gradient direction for clustering in the framework of multi-task learning, which shares information within the same clusters. In FedAMP \citep{FedAMP}, adaptively exchanges information among clients according to the similarity of their personalized models.  \cite{PFLG} shares information among neighboring clients based on graph-structured relationships, while updating the global model and local models. These PFL methods rely on a shared global model or infer client relationships from model updates or parameter similarities. In contrast, instead of estimating relationships between clients based on model updates, we incorporate geographical proximity as prior information through graph-based regularization, thereby promoting the sharing of information between geographically close institutions whilst maintaining personalized models.

\begin{figure}[htbp]
\begin{center}
\includegraphics[scale=0.6]{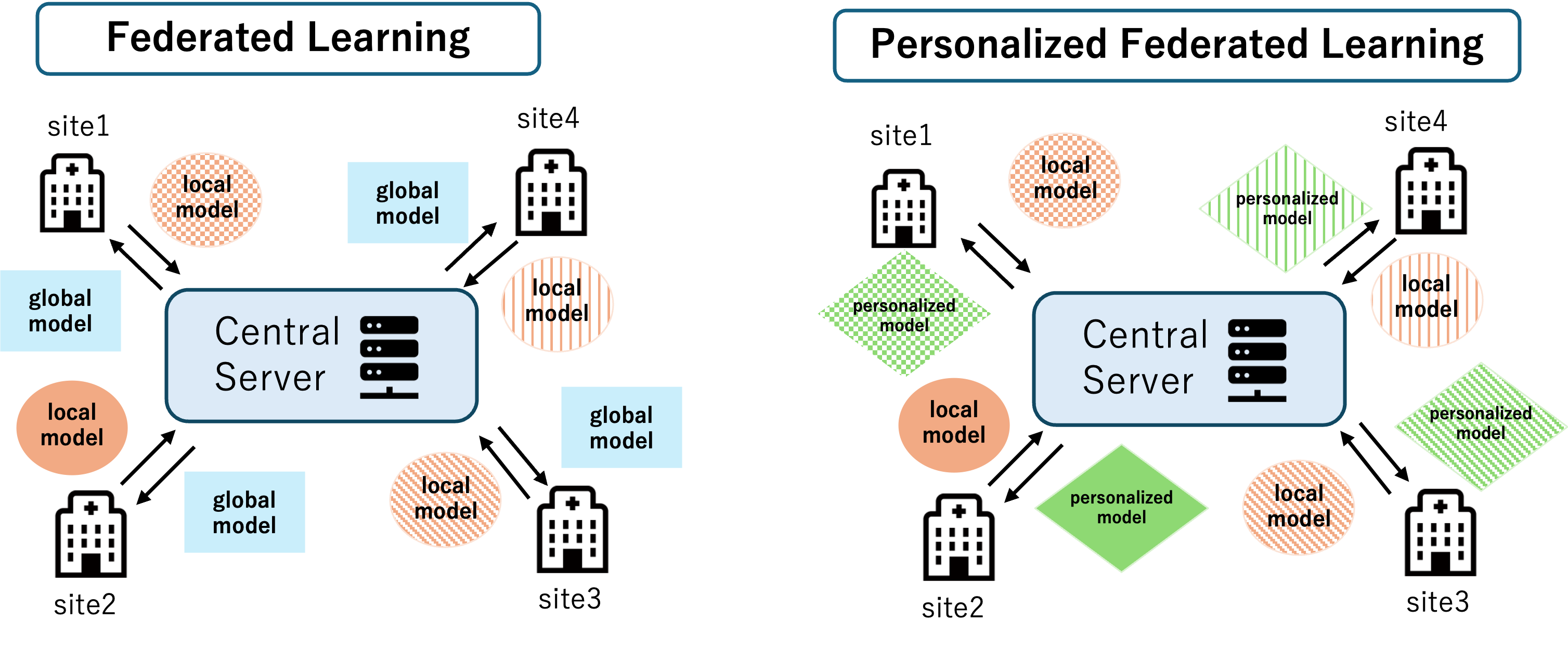}
\caption{Images of federated learning and personalized federated learning.}
\label{fl_pfl}
\end{center}
\end{figure}

\subsection{AUC maximization}
\label{subsec:AUC_max}

\noindent
{\bf Definition of AUC}

The ROC curve depicts the relationship between the true positive rate (TPR) and false positive rate (FPR) as the threshold for the scoring function varies. 
For the scoring function $s(x)$, the TPR and FPR when using threshold $\nu \in \mathbb{R}$ are defined as $TPR(\nu) = P(s(X)\geq \nu|Y=1)$ and $FPR(\nu) = Pr(s(X)\geq \nu | Y=0)$, respectively.
 The ROC curve is then given as the set of 
$(FPR(\nu), TPR(\nu))$ obtained as $\nu$ varies.

Let the target variable $y \in \{0,1 \}$ be a binary variable representing the observed true label, where a positive instance is denoted by $y=1$ and a negative instance by $y=0$. Let the covariates corresponding to the positive and negative examples be denoted by $\bm{x}^+ \in \mathbb{R}^p$ and $\bm{x}^- \in \mathbb{R}^p$, respectively, and assume these are independent random variables following probability distributions $P^+$ and $P^-$, respectively. In this case, the AUC is defined as the probability that a positive subject achieves a higher score than a negative subject. Let $f(\bm{x})$ be the scoring function and $I(\cdot)$ be the indicator function. Then AUC as the probability is defined as follows:
\begin{align}
{\rm AUC}(f) = P\big((f(\bm{x}^+) > f(\bm{x}^-)\big) = \mathbb{E}_{\bm{x}^+ \sim P^+, \bm{x}^- \sim P^-}\big[I\big(f(\bm{x}^+) > f(\bm{x}^-)\big)\big].
\label{auc_max}
\end{align}
In the case of actual data, where the covariates for positive cases are denoted by $\bm{x}_i^+$ ($i=1, 2, \dots, n$) and those for negative cases by $\bm{x}_j^-$ ($j=1, 2, \dots, m$), the AUC is expressed as the following:
\begin{align}
\mathrm{AUC}
&= \frac{1}{nm}
\sum_{i=1}^{n}
\sum_{j=1}^{m}
I(f(\bm{x}^+_i)>f(\bm{x}^-_j))
\label{auc_func}
\end{align}
where $f(\bm{x}): \mathbb{R}^p \rightarrow \mathbb{R}$ is the function to predict the parameters, and we assume a linear scoring function $f(\bm{x})= \bm{\beta}^\top \bm{x}$, where $\bm{\beta}$ signifies the parameter to calculate the score. 
$I(\cdot)$ is the indicator function returning $1$ if $f(\bm{x}^+_i)>f(\bm{x}^-_j)$. 
Eq. (\ref{auc_func}) corresponds to the standardized Mann-Whitney U-statistic \citep{Ustat}, which represents the probability that a randomly selected positive sample will be assigned a higher score than a randomly selected negative sample.

Under the scoring model, AUC can be expressed as a pairwise comparison of score differences:
\begin{align}
\mathrm{AUC}
&= \frac{1}{n m}
\sum_{i =1}^{n}
\sum_{j =1}^{m}
I\!\left(
\bm{\beta}^\top \bm{x}^+_{i} > \bm{\beta}^\top\bm{x}^-_{j}
\right) \nonumber\\
&= \frac{1}{n m}
\sum_{i =1}^{n}
\sum_{j =1}^{m}
I \left(
\bm{\beta}^\top (\bm{x}^+_{i} - \bm{x}^-_{j}) > 0
\right) \nonumber\\
&=  \frac{1}{n m}
\sum_{i =1}^{n}
\sum_{j =1}^{m}
I \left(
f(\bm{x}_i^+) - f(\bm{x}_j^-) >0
\right) 
\label{auc_func2}
\end{align}
Eq. (\ref{auc_func2}) shows that AUC maximization can be interpreted as a ranking problem based on pairwise differences.

\noindent
{\bf Surrogate function of AUC maximization
}

As shown in Eq. (\ref{auc_func2}), AUC depends on the discontinuous indicator function, which is not differentiable and not suitable for gradient-based optimization. 
Therefore, it is common to approximate it using a surrogate function. Various surrogate loss functions have been proposed  \citep{Yuan2021, TIAN20111691}, and in this study, we employ the logistic loss \citep{sulam_logis}. This provides a smooth approximation while preserving the ranking structure and enable efficient optimization.

For a pair $(i,j)$, the logistic-based surrogate can be described as:
\begin{align*}
\ell\!\left(\bm{\beta}_; \bm{x}_{i}^{+}, \bm{x}_{j}^{-}\right)
= 
\log \big(1+\exp[-(\bm{\beta}^\top \bm{x}_{i}^{+} - \bm{\beta}^\top \bm{x}_{j}^{-} )-q] \big)
\end{align*}
where $q \ (q \geq 0)$ is the margin constant value. Here, $\bm{\beta}$ signifies the parameter to calculate the score. The surrogate loss corresponding to AUC is defined as follows:
\begin{align}
 \frac{1}{n m}
\sum_{i =1}^{n}
\sum_{j =1}^{m}
\log \big(1+\exp[-(\bm{\beta}^\top \bm{x}_{i}^{+} - \bm{\beta}^\top \bm{x}_{j}^{-} )-q] \big).
\label{auc_surrogate}
\end{align}
This surrogate function preserves the ranking structure of each pair while providing a differentiable approximation of the AUC objective function. 
This enables efficient optimization using gradient-based methods.

\section{Geographically regularized AUC-maximizing personalized federated learning (GrAUC-PFL)}
\label{sec:proposed_method}

\subsection{Framework of GrAUC-PFL}
\label{subsec:framework}

\begin{figure}[htbp]
\begin{center}
\includegraphics[width=0.8\linewidth]{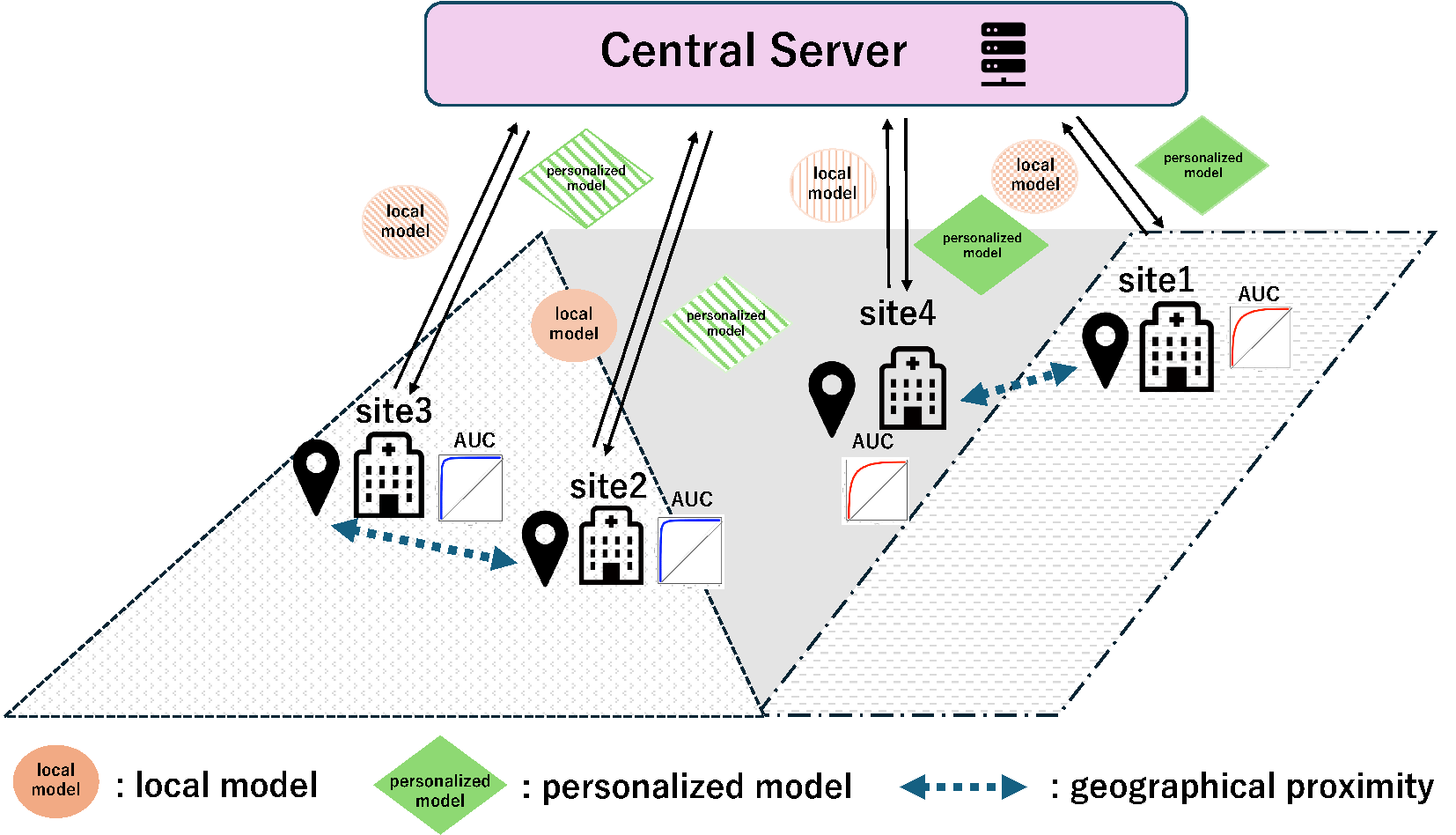}
\caption{Image of the framework of Personalized Federated Learning considering the geographical proximity.}
\label{pfl_frw}
\end{center}
\end{figure}

We first introduce the framework of GrAUC-PFL. 
Let $z=1, 2, \dots, Z$ denote institutions, each of which is treated as a client with the FL framework and has its own local dataset consisting of outcomes and covariates for the disease, along with geographical information such as latitude and longitude. Because raw data cannot be shared externally, each institution aims to estimate its own model parameter $\bm{\beta}_z$ based solely on the relevant local data.

The objective of this study is to utilize information from multiple institutions and directly maximize AUC, as a measure of discriminative performance, through a unified PFL framework with geographical proximity. To account for this spatial structure, we assume that proximate institutions tend to share similar risk factors, and incorporate this similarity into the model through a regularization structure that encourages similar parameters among nearby institutions. Specifically, we adopt an $L_1$-type penalty term based on  differences between the coefficient of selected pairs of institutions, as in spatially clustered coefficients (SCC) \citep{SCC}. The pairs of institutions are determined by the edges of a graph constructed based on their geographical locations. If the difference between a pairs of coefficients is non-zero, the edge represents a boundary; otherwise, the institutions are regarded as belonging to the same cluster. To construct the graph, a minimum spanning tree (MST) is employed to select the edge as in the SCC. In this method, $Z-1$ edges connecting all institutions are constructed at minimum. This significantly reduces  computational cost compared to a complete graph, which has $Z(Z-1)/2$ edges. This formulation is also related to fused lasso regularization \citep{fused_lasso}, allowing existing algorithms to be readily applied. Moreover, when the number of subject in a institution is small, estimation of the coefficients may be unstable. By incorporating information from neighboring institutions, the estimation accuracy can be improve compared to independent estimation. 

Next, we describe the communication scheme between the server and local participating institutions for parameter updates. We adopt the the PFL framework proposed by \cite{perFL-RSR}. In this framework, the server and each institution have distinct roles. Each institution maintains its own local model and updates its model parameters using only local data. Meanwhile, the server coordinates the communication among institutions and enforces the graph-based constraints that capture the relationships between institutions.  During each iteration, only model-related information is exchanged between the server and the institutions, while patient-level data remain local. By iteratively exchanging this information, each institution can estimate a personalized model that reflects both its local data and the information shared across related institution. The proposed framework is illustrated in Figure \ref{pfl_frw}.

\subsection{Optimization problem of GrAUC-PFL}
\label{subsec:opt_prob}

In this section, we define the optimization problem of GrAUC-PFL. 
For institutions $z = 1, \dots, Z$, each institution has its own local dataset $\{(\bm{x}_{hz}, y_{hz})\}_{h=1}^{N_z}$, where $\bm{x}_{hz} \in \mathbb{R}^p$ is a covariate vector and $y_{hz} \in \{0,1\}$ is a binary outcome. For the covariates at institution $z$, $\{\bm{x}_{iz}^{+}\}_{i=1}^{n_z}$ and $\{\bm{x}_{jz}^{-}\}_{j=1}^{m_z}$ represent the positive and negative samples, respectively. Here, $N_z=n_z+m_z$. Also, we consider the geographical proximity between institutions, defined through a distance measure between them, and incorporate this information into the model to encourage similarity between nearby coefficient vectors. In this study, we use the distance computed from latitude and longitude, and those in the institution $z$ denote $\bm{s}_z = ({\rm lon}_z,  {\rm lat}_z) \in \mathbb{R}^2$, where ${\rm lon}_z$ denote longitude and ${\rm lat}_z$ latitude. 

We consider the linear scoring function:
\begin{align*}
f_z(\bm{x}_z) = \bm{\beta}(\bm{s}_z)^\top\bm{x}_z. 
\end{align*}
To simplify the notation, we shall denote $\bm{\beta}(\bm{s}_z)$ by $\bm{\beta}_z$. The AUC at institution $z$ can be interpreted as the probability that a positive sample is ranked higher than a negative one.
\begin{align*}
\mathrm{AUC}_z
&= \frac{1}{n_z m_z}
\sum_{i =1}^{n_z}
\sum_{j =1}^{m_z}
I\!\left(
f_z(\bm{x}^+_{iz}) > f_z(\bm{x}^-_{jz})
\right)\\
&= \frac{1}{n_z m_z}
\sum_{i =1}^{n_z}
\sum_{j =1}^{m_z}
I\!\left(
\bm{\beta}_z^\top (\bm{x}^+_{iz} - \bm{x}^-_{jz}) > 0
\right)
\end{align*}
However, because the indicator function is non-differentiable, it is difficult to optimize this objective directly. Therefore, we replace this with a pairwise logistic loss as a surrogate function at institution $z$.
\begin{align*}
\ell_z\!\left(\bm{\beta}_z\right)
=
\frac{1}{n_z m_z}
\sum_{i =1}^{n_z}
\sum_{j =1}^{m_z}
\log \left( 1 + \exp [-(\bm{\beta}_z^\top\bm{x}_{iz}^{+}-\bm{\beta}_z^\top\bm{x}_{jz}^{-}) - q  ] \right)
\end{align*}
where $q \ge 0$ is the margin parameter. Using this surrogate loss, the optimization problem of GrAUC-PFL is formulated as follows:
\begin{align}
\mathcal{L}(\bm{\beta})
=  \frac{1}{Z}\sum_{z=1}^Z
\frac{1}{n_z m_z} \sum_{i=1}^{n_z} \sum_{j=1}^{m_z} 
\log \big(1+\exp[-(\bm{\beta}_z^\top \bm{x}_{iz}^{+} - \bm{\beta}_z^\top \bm{x}_{jz}^{-} )-q] \big)& \nonumber\\
+ \lambda 
\sum_{(u,v)\in E} w_{uv} \|\bm{\beta}_u - \bm{\beta}_v  \|_1
\label{prop_obj}
\end{align}
where $\lambda (>0)$ is a tuning parameter and $\| \cdot \|_1$ denotes $L_1$ norm. The first term of Eq. (\ref{prop_obj}) is to optimize the surrogate function of AUC maximization. The parameter $\bm{ \beta}_z$ is updated on each local institution and consolidated on the server, and updated as $\bm{\beta}$. The second term is a weighted fused penalty imposed on the differences between institution-specific coefficient vectors. $\{u,v\} \in E$ represents an edge connecting institutions $u$ and $v$, where $E$ is the edge set constructed from MST. This penalty encourages geographically close institutions to have similar coefficient vectors, and may reduce the difference between certain pairs to exactly zero. Consequently, the institutions connected through the graph can share identical parameter vectors, yielding a clustering structure among institutions in terms of their regression coefficients. $w_{uv}$ is a distance-based weight reflecting the geographical proximity between institutions $u$ and $v$ using a Gaussian kernel. The distance between the institutions $u$ and $v$ is defined as $m^\dagger_{uv} = d(\bm{s}_u, \bm{s}_v)$, where $d(\cdot, \cdot)$ is the distance metric. For this metric calculation, we use the Haversine formula. 
Based on this distance, we define the weight as
\begin{align}
w_{uv} = \exp \bigg( - \frac{m^{\dagger2}_{uv}}{2 \phi^2} \bigg).
\label{w_weight}
\end{align}
$\phi$ is a bandwidth parameter determined by $w_{dist} \cdot \phi_{base}$, where $\phi_{base}= {\rm median }\{ m^\dagger_{uv}: \{u,v\} \in E \}$. 
The weight $w_{uv}$ decreases as the distance $m^\dagger_{uv}$ increases, assigning larger weights to geographically closer institutions, consistent with the principle of spatial weighting that observations at closer distances receive greater influence\citep[e.g.][]{GWR}. The bandwidth $\phi$ determines how quickly the weight decays with distance: smaller $\phi$ lead to more rapidly decreasing weights, while larger values result in more uniform weighting across institutions. Here, $\phi_{base}$ is set to the median of the pairwise distances over the edges in the MST. This is motivated by the median heuristic commonly used in the bandwidth of the Gaussian kernel and is widely adopted in kernel methods \citep[e.g.][]{medheu, gertton}. This provides a robust estimate of the typical distance scale, as MST captures relationship between institutions in local neighborhood. The MST-based median provides a scale that is consistent with the graph structure used in the regularization. On the other hand, since the appropriate scale depends on the data distribution, the scaling parameter $w_{dist} (\geq 0)$ is selected via cross-validation.

\subsection{Reformulation of GrAUC-PFL}
\label{subsec:reform}

In the framework of PFL, parameter $\bm{\beta}_z$ needs to be updated at each institution by using its own local data. To update these parameters by distributed learning efficiently, we adpot the ADMM. However, the second term couples institution-specific parameters, preventing the update from being decomposed across institutions and making the problem difficult to solve. To address this issue, it needs to reformulate the pairwise differences between coefficients of connected institutions into an alternative representation. First, the pairwise differences between the coefficients are expressed as
\begin{equation}
\left[ 
\begin{array}{c}
\bm{\beta}_{u_1} -\bm{\beta}_{v_1} \\
\bm{\beta}_{u_2} -\bm{\beta}_{v_2}\\
\vdots\\
\bm{\beta}_{u_{|E|}} -\bm{\beta}_{v_{|E|}}\\
\end{array}
\right] 
\in \mathbb{R}^{|E|p}, \ {\rm where} \ {(u_o,v_o)\in E} \ (o=1, 2, \dots, |E|).
\label{prop_2nd_diff}
\end{equation}
As the difference structure of the coefficients is defined based on MST, we begin by using the edges of the graph to represent the relationship between institutions. From the graph $(V^\dagger, E)$, where $V^\dagger$ denotes a set of $Z$ locations, and $E$ denotes the set of edges. After computing pairwise distances based on the coordinate data, an adjacency matrix is constructed to represent geographical proximity, and a graph is defined accordingly. We use the MST to determine the edge set $E$. MST connects all institutions while minimizing the total pairwise distance and results in a sparse structure with $|E|=Z-1$ edges. 
Based on $E$, the incidence matrix $\bm{D} = (d_{oz}) \in \mathbb{R}^{|E| \times Z}$ is defined as follows:
\begin{align}
 \begin{cases}
  d_{oz} =  +1 \ (z=u_o)\\
  d_{oz} =  -1 \ (z=v_o)\\
  d_{oz} = 0 \ \ ({\rm otherwise}) \\
\end{cases}, \quad (o=1, 2, \dots, |E|)
\label{incidence}
\end{align}
Next, since $\bm{\beta}_z$ is a institution-specific vector, all coefficient vectors are concatenate into a single vector:
\begin{align*}
\bm{\beta}=vec(\bm{\beta}_1,\bm{\beta}_2, \dots, \bm{\beta}_Z) \in \mathbb{R}^{Zp}.
\end{align*}
where $vec(\cdot)$ represents a vec operator.
This formulation allows the difference between each pair of coefficients to be expressed using linear operators. Here, the incidence matrix $\bm{D}$ defines a scalar difference operator, while $\bm{\beta}$ is vector-valued. Therefore, $\bm{D}$ must be extended to act on vector-valued parameters. To achieve this, the Kronecker product is applied with $\bm{I}_p$ and denoted by $\bm{\Omega}$:
\begin{align*}
\bm{D} \otimes \bm{I}_p \in \mathbb{R}^{|E|p\times Zp} = \bm{\Omega}.
\end{align*}
where $\otimes$ is the Kronecker product and $\bm{I}_p$ is the identity matrix.

Using $\bm{\Omega}$, the pairwise differences between institution-specific coefficient vectors can be expressed as
\begin{equation*}
 (\bm{D} \otimes \bm{I}_p)\bm{\beta} 
 =
\bm{\Omega}\bm{\beta}
\end{equation*}
where the resulting vector is obtained by summing the differences along all edges. 
$\bm{\Omega}\bm{\beta}$ is equivalent to Eq. (\ref{prop_2nd_diff}). 
Therefore, the objective function in Eq. (\ref{prop_obj}) can be rewritten as
\begin{align}
\min_{ \{ \bm{\beta} \}} 
 \frac{1}{Z}\sum_{z=1}^Z
\frac{1}{n_z m_z} \sum_{i=1}^{n_z} \sum_{j=1}^{m_z} 
\log \big(1+\exp[-(\bm{\beta}_z^\top \bm{x}_{iz}^{+} - \bm{\beta}_z^\top \bm{x}_{jz}^{-} )-q] \big) 
+ \lambda  \|\bm{W} \bm{\Omega}\bm{\beta} \|_1
\label{prop_2_omega}
\end{align}
where $\bm{W} =\mathrm{diag}(\bm{w}) \otimes \bm{I}_p \in \mathbb{R}^{|E|p \times |E|p}$ is a block-diagonal weight matrix. From here, to simplify the notation, we shall represent the subscripts of $w_{uv}$ using a single index and express it as a vector as $\bm{w}=(w_1, w_2, \dots, w_{|E|})$. Furthermore, to apply Eq. (\ref{prop_2_omega}) to the framework of PFL and utilize the ADMM algorithm, the objective function in Eq. (\ref{prop_2_omega}) can be rewritten with the constraint as
\begin{align}
\min_{ \{ \bm{\beta}, \bm{\delta} \}} 
 \frac{1}{Z}\sum_{z=1}^Z
\frac{1}{n_z m_z} \sum_{i=1}^{n_z} \sum_{j=1}^{m_z} 
\log \big(1+\exp[-(\bm{\beta}_z^\top \bm{x}_{iz}^{+} - \bm{\beta}_z^\top \bm{x}_{jz}^{-} )-q] \big) 
+ \lambda  \| \bm{W}\bm{\delta} \|_1, \quad \mathrm{s.t.}  \ \bm{\Omega}\bm{\beta}=\bm{\delta}.
\label{prop_2}
\end{align}
$\bm{\delta}$ denotes the vector whose elements are the difference between the coefficients $u$ and $v$ for each variable. This reformulation enables the use of the ADMM because it separates the coupled terms and facilitates tractable optimization. Specifically, the introduction of the auxiliary variable allows the optimization problem to be decomposed into subproblems. In the framework of the PFL, the updates for $\bm{\beta}$ must be performed locally because they require data from each institution, whereas the global variables, including $\bm{\delta}$, are updated on the server side. Therefore, the communication overhead between the server and each institution associated with the updates can be reduced. Eq. (\ref{prop_2}) can be solved based on the corresponding Lagrangian function, which can be defined as follows: 
\begin{align}
\mathcal{L}_\rho(\bm{\beta}, \bm{\delta} ,\bm{\gamma} )
=&
 \frac{1}{Z}\sum_{z=1}^Z
\frac{1}{n_z m_z} \sum_{i=1}^{n_z} \sum_{j=1}^{m_z} 
\log \big(1+\exp[-(\bm{\beta}_z^\top \bm{x}_{iz}^{+} - \bm{\beta}_z^\top \bm{x}_{jz}^{-} )-q] \big)\nonumber \\
&+ \lambda \sum_{o=1}^{|E|}w_o \| \bm{\delta}_o \|_1 
+ \bm{\gamma}^\top(\bm{\Omega}\bm{\beta}- \bm{\delta}) 
+ \frac{\rho}{2} \|\bm{\Omega}\bm{\beta}- \bm{\delta}\|_2^2
\label{obj_lag}
\end{align}
where $\bm{\delta}= (\bm{\delta}_1, \bm{\delta}_2,\dots, \bm{\delta}_{|E|})^\top = (\delta_\xi), \  (\xi=1,2,\dots, |E|p)$ and $\bm{\delta}_o =(\delta_{o1}, \delta_{o2}, \dots, \delta_{o p})^\top \in \mathbb{R}^p$ is different vector corresponding to $o$, $\bm{\gamma} \in \mathbb{R}^{|E|p}$ is the Lagrangian multiplier and $\rho (\rho>0)$ is the tuning parameter. $\|\cdot \|_2^2$ is the Euclidean norm.

\subsection{Update \texorpdfstring{$\bm{\beta}$}{beta}}
\label{subsec:update_beta}

In this subsection, we explain how to derive the updated formula of $\bm{\beta}$. As $\bm{\beta}$ depends on the local data of each institution, within the PFL framework, the parameters must be updated at each institution. In the PFL framework proposed in \cite{perFL-RSR}, the updated formula of $\bm{\beta}$ can be decomposed into a server-side term that can be evaluated centrally and a client-specific term that depends on local data. This procedure enables parameter estimation that accounts for institutional heterogeneity while reducing computational costs and preserving data privacy. 
The updated formula of $\bm{\beta}$ is as follows:
\begin{align}
&\bm{\beta}^{(t+1)} 
= \bm{\beta}^{(t,g)}
- r^{-1} \tau \nabla f(\bm{\beta}^{(t)}), \label{upd_beta_all}\\
 {\rm where} \ & \bm{\beta}^{(t,g)}
= r^{-1}\bm{H}\bm{\beta}^{(t)}
- r^{-1} \tau
\big[ 
- \rho \bm{\Omega}^\top \bm{\delta}^{(t)}
+ \bm{\Omega}^\top \bm{\gamma}^{(t)}
\big].
\label{upd_beta_decomp}
\end{align}
$\bm{\beta}^{(t,g)}$ is the global component which depends on the regularization term and ADMM variables. $\bm{\beta}^{(t)}$ is $\bm{\beta}$ at $t$th step. $\nabla f(\bm{\beta})$ is the gradient of $f(\bm{\beta})$, which represents the first term of Eq. (\ref{obj_lag}). $\tau \ (\tau>0)$ is a step-size parameter that controls the curvature of the quadratic approximation. $\bm{H}$ is a positive definite matrix, which can be defined as follows:
\begin{align}
\bm{H}= r\bm{I} - \rho \tau \bm{\Omega}^\top \bm{\Omega}. 
\label{H_set}
\end{align}
$r \ (r>0)$ is a scaling parameter for $\bm{H}$, which is set as 
\begin{align}
r 
> \rho \tau \varsigma_{\max}(\bm{\Omega}^\top \bm{\Omega})
+ \max \bigg(\frac{\tau\mu }{2}, 1 \bigg).
\label{r_setting}
\end{align}
First, We explain the Lipschitz continuity of the gradient of $f(\bm{\beta})$ to derive this majorizing function. 

\begin{lem} The gradient of $f(\bm{\beta})$ is Lipschitz continuous with $\mu = \max_{z \in \{ 1,2,\dots, Z\} }(\mu_z)$, where $\mu_z = \varsigma_{\max} \big(
\frac{1}{4n_zm_z}
\sum_{i=1}^{n_z}
\sum_{j=1}^{m_z}
\bm{d}_{ijz}\bm{d}_{ijz}^{\top} \big)$, in which $\bm{d}_{ijz}=\bm{x}_{iz}^+ - \bm{x}_{jz}^-, \ (i=1,\dots, n_z; j=1, \dots, m_z)$. Here, $\varsigma_{\max}(\bm{O})$ describes the largest eigenvalue of matrix $\bm{O}$.
\end{lem}
The proof of Lemma 1 is shown in Appendix \ref{secA1}. 
Using the Lipschitz continuity of $\nabla f(\bm{\beta})$, which denotes the gradient of $f(\bm{\beta})$,  established in Lemma 1, the update of $\bm{\beta}$ is obtained by approximating the first term of Eq. (\ref{obj_lag}) with a quadratic function, while retaining the linear and quadratic penalty terms in their original form. 
When $\bm{H}$ is a positive definite matrix, the updated formula of $\bm{\beta}$ can be derived based on the following function:
\begin{align}
\tilde{\mathcal{L}}_\rho (\bm{\beta}; \bm{\beta}^{(t)}, \bm{\delta}^{(t)}, \bm{\gamma}^{(t)})
=f(\bm{\beta}^{(t)}) 
+ \nabla f(\bm{\beta}^{(t)})^\top(\bm{\beta}- \bm{\beta}^{(t)})
+ \frac{1}{2 \tau}(\bm{\beta} - \bm{\beta}^{(t)})^\top{\bm{H}}(\bm{\beta} - \bm{\beta}^{(t)})
+ \bm{\gamma}^\top\bm{\Omega}\bm{\beta}
+ \frac{\rho}{2} \|\bm{\Omega}\bm{\beta}- \bm{\delta}\|_2^2. 
\label{obj_beta_mm1}
\end{align}
where $\bm{\beta}^{(t)}$ is $\bm{\beta}$ at $t$th step. The following Theorem 2 shows that the results of the updated $\bm{\beta}$ yields a sufficient decrease in the augmented Lagrangian function of Eq. (\ref{obj_lag}). Here, Theorem 2 states that Eq. (\ref{obj_beta_mm1}) is the majorizing function \citep{MMalgorithm} of Eq. (\ref{obj_lag}) by choosing $r$ appropriately. Theorem 2 can be derived in the same manner as Lemma 1 in \cite{perFL-RSR}. 
\begin{thm}
For an update of $\bm{\beta}$, the difference in the augmented Lagrangian function satisfies the following condition when $r> \rho \tau \varsigma_{\max}(\bm{\Omega}^\top\bm{\Omega})+\max( \frac{\tau\mu}{2}, 1)$:
\begin{align}
\mathcal{L}_\rho (\bm{\beta}^{(t+1)}, \bm{\delta}^{(t)}, \bm{\gamma}^{(t)})
- \mathcal{L}_\rho (\bm{\beta}^{(t)}, \bm{\delta}^{(t)}, \bm{\gamma}^{(t)}) 
\leq 
- \bigg[
\frac{\varsigma_{\min}(\bm{H})}{\tau}
+ \frac{\rho \cdot \varsigma_{\min}(\bm{\Omega}^\top \bm{\Omega})}{2}
- \frac{\mu}{2}
\bigg]
\| \bm{\beta}^{(t+1)}- \bm{\beta}^{(t)}\|_2^2.
\label{perfl_lemma1}
\end{align}
where $\varsigma_{max}(\bm{O})$ and $\varsigma_{min}(\bm{O})$denote the largest and smallest eigenvalue of matrix $\bm{O}$, respectively. This property follows from the uniform boundedness of the Hessian of the logistic loss.
\end{thm}

The proof of Theorem 2 is in Appendix \ref{secB}. 
Using Eq. (\ref{obj_beta_mm1}), the updated formula of $\bm{\beta}$ is defined as
\begin{align}
\bm{\beta}^{(t+1)} = r^{-1} \bm{H} \bm{\beta}^{(t)}
- r^{-1} \tau\big[ 
\nabla f(\bm{\beta}^{(t)})
- \rho \bm{\Omega}^\top \bm{\delta}^{(t)}
+ \bm{\Omega}^\top \bm{\gamma^{(t)}}
\big].
\label{upd_beta2}
\end{align}
When $\bm{H}$ is just a positive definite, the updated formula of $\bm{\beta}$ includes the inverse of $Zp \times Zp$ matrix. 
Compared with direct updates with the inverse matrix, whose computational cost is $O\{ (Zp)^3\}$, the update of Eq. (\ref{upd_beta2}), avoiding matrix inversion requires matrix-vector multiplications. Therefore, the computational cost is reduced to $O\{(Zp)^2 \}$. 
Here, in Eq. (\ref{upd_beta2}), only $\nabla f (\bm{\beta}^{(t)})$ depends on the data at each institution, while the remaining components are handled on the server. Therefore, to reduce the computational load, \cite{perFL-RSR} decomposed Eq. (\ref{upd_beta2}), as Eq. (\ref{upd_beta_all}) and Eq. (\ref{upd_beta_decomp}). $\bm{\beta}^{(t,g)}$ is the global component which depends on the regularization term and ADMM variables. This is computed at the central server. By contrast, the gradient of the loss function $\nabla f(\bm{\beta})$ is the local component. 
$\nabla f(\bm{\beta})$ is can be decomposed into $\nabla f_z(\bm{\beta}_z)$, which can be computed independently at each institution. 
As a result, updating $\bm{\beta}^{(t,g)}$ is performed on the server first, then $\bm{\beta}^{(t+1)}$ is computed at each institution as $\bm{\beta}_z^{(t+1)}$. 

\noindent
{\bf Update on server}

To update $\bm{\beta}$, Eq. (\ref{upd_beta_decomp}) is updated first on the server. After the global component is updated, the server distributes $\bm{\beta}^{(t,g)}_z$ to each institution to update $\bm{\beta}_z$.

\noindent
{\bf Update at local institutions}

Each local institution downloads $\bm{\beta}^{(t,g)}_z$ and updates $\bm{\beta}_z$. 
Let $f_z= \frac{1}{n_z m_z}
\sum_{i=1}^{n_z} \sum_{j=1}^{m_z}
\log \big(1+\exp[-(\bm{\beta}_z^\top \bm{d}_{ijz} )-q] \big)$, then $\bm{\beta}_z$ can be updated by the following:
\begin{align}
\bm{\beta}_z^{(t+1)} = \bm{\beta}_z^{(t,g)} - r^{-1} \tau \nabla f_z (\bm{\beta}_z^{(t)}) 
\label{upd_beta_term1}
\end{align}
where
\begin{align*}
\nabla f_z (\bm{\beta}_z)
=\frac{1}{n_z m_z} 
\sum_{i=1}^{n_z} \sum_{j=1}^{m_z}
\big(- \bm{d}_{ijz} \sigma(- (\bm{\beta}_z^\top \bm{d}_{ijz}+q)) \big). 
\end{align*}
Here, $\sigma(a) = 1/(1+\exp^{-a}).$ 
Updated $\bm{\beta}_z$ is sent back to the server and merged as $\bm{\beta}^{(t+1)}$ by updating Eq. (\ref{upd_beta_decomp}).

\begin{remark}
Unless $\bm{H} \neq  r\bm{I} - \rho \tau \bm{\Omega}^\top \bm{\Omega}$, even if $\bm{H}$ is a positive semidefinite matrix, then the updated formula for $\bm{\beta}$ in Eq. (\ref{obj_beta_mm1}) involves the inverse matrix. Therefore, maintaining the inverse matrix on the central server can become a significant burden, particularly when the number of variables or the number of participating institutions is large. From the right-hand side of Eq. (\ref{obj_beta_mm1}), the terms related to $\bm{\beta}$ are
\begin{align}
\nabla f(\bm{\beta}^{(t)})^\top(\bm{\beta}- \bm{\beta}^{(t)})
+ \frac{1}{2 \tau}(\bm{\beta} - \bm{\beta}^{(t)})^\top{\bm{H}}(\bm{\beta} - \bm{\beta}^{(t)})
+ \bm{\gamma}^\top\bm{\Omega}\bm{\beta}
+ \frac{\rho}{2} \|\bm{\Omega}\bm{\beta}- \bm{\delta}\|_2^2
+ \mathrm{const},
\label{terms_beta}
\end{align}
where $\mathrm{const}$ is a constant value. Differentiate Eq. (\ref{terms_beta}) by $\bm{\beta}$ and set it as $\bm{0}$;
\begin{align*}
\nabla f(\bm{\beta}^{(t)})
+ \frac{1}{\tau}{\bm{H}}(\bm{\beta} - \bm{\beta}^{(t)})
+ \bm{\Omega}^\top\bm{\gamma}
+ \rho \bm{\Omega}^\top(\bm{\Omega}\bm{\beta}- \bm{\delta})
&=\bm{0}\\
\Longleftrightarrow \quad
\nabla f(\bm{\beta}^{(t)}) 
+ \tau^{-1}\bm{H}\bm{\beta}
- \tau^{-1}\bm{H}\bm{\beta}^{(t)}
+ \bm{\Omega}^\top\bm{\gamma}
+ \rho \bm{\Omega}^\top\bm{\Omega}\bm{\beta}
-\rho \bm{\Omega}^\top\bm{\delta}
&=\bm{0}\\
\Longleftrightarrow 
\bigg(\tau^{-1}\bm{H} + \rho \bm{\Omega}^\top\bm{\Omega}\bigg)
\bm{\beta} 
=
\tau^{-1}\bm{H}\bm{\beta}^{(t)}
-\nabla f(\bm{\beta}^{(t)})
- \bm{\Omega}^\top\bm{\gamma}
+ \rho \bm{\Omega}^\top\bm{\delta} 
\end{align*}
From Eq. (\ref{terms_beta}), we can obtain the updated formula of $\bm{\beta}$ as
\begin{align}
\bm{\beta}^{(t+1)} =& (\tau^{-1}\bm{H}+\rho \bm{\Omega}^\top \bm{\Omega})^{-1}
 \big[ 
\tau^{-1}\bm{H}\bm{\beta}^{(t)} 
- \nabla f(\bm{\beta}^{(t)})
+ \rho \bm{\Omega}^\top \bm{\delta}^{(t)}
- \bm{\Omega}^\top \bm{\gamma^{(t)}}
\big].
\label{upd_beta1}
\end{align}
Therefore, in \cite{perFL-RSR}, $\bm{H}$ is defined as Eq. (\ref{H_set}), which allows calculation without requiring the inverse matrix. 
To obtain this,  we substitute Eq. (\ref{H_set}) into the first term of Eq. (\ref{upd_beta1}):
\begin{align*}
\tau^{-1}\bm{H}+\rho\bm{\Omega}^\top\bm{\Omega}=\tau^{-1}(r\bm{I}-\rho \tau\bm{\Omega}^\top\bm{\Omega})+ \rho\bm{\Omega}^\top\bm{\Omega}= r \tau^{-1}\bm{I},
\end{align*}
then, 
\begin{align*}
(\tau^{-1}\bm{H}+\rho \bm{\Omega^\top \Omega})^{-1}= r \tau^{-1}\bm{I}.
\end{align*}
Therefore, with $\bm{H}$ as in Eq. (\ref{H_set}), the updated formula of $\bm{\beta}$, the inverse matrix is reduced to a scalar multiple of the identity matrix, thereby avoiding the need to compute the inverse matrix directly.
\end{remark}

\subsection{Update \texorpdfstring{$\bm{\delta}$}{delta}}
\label{subsec:update_delta}

Next, $\bm{\delta}$ is updated on the server after $\bm{\beta}$ has been updated. The updated formula of $\bm{\delta} = (\delta_\xi)$ is as follows:
\begin{align}
\delta_\xi^{(t+1)}
=
S_\psi \big(\nu^{(t)}_\xi \big),
\quad (\xi= 1, 2, \dots, |E|p) 
\label{delta_upd}
\end{align}
where $S_\psi(\cdot)$ denotes soft thresholding operator. $\psi = \frac{\lambda \tilde{w}_\xi}{\rho}$, $\bm{\nu}^{(t)} = (\nu^{(t)}_\xi)$ is defined as $\bm{\Omega}\bm{\beta}^{(t)} + \rho^{-1} \bm{\gamma}$. $\tilde{w}_\xi$ denotes the edge weight $w_{uv}$ corresponding to $\delta_\xi$.
Now we explain the procedure for deriving Eq. (\ref{delta_upd}). 
The terms associated with $\bm{\delta}$ in Eq. (\ref{obj_lag}) are
\begin{align}
 \lambda  \sum_{o=1}^{|E|}w_o \| \bm{\delta}_o \|_1  
+ \bm{\gamma}^\top(\bm{\Omega}\bm{\beta}- \bm{\delta}) 
+ \frac{\rho}{2} \|\bm{\Omega}\bm{\beta}- \bm{\delta}\|_2^2.
\label{obj_delta}
\end{align}

Transform Eq. (\ref{obj_delta}):
\begin{align}
\rho^{-1}\bm{\gamma}^\top(\bm{\Omega}\bm{\beta}- \bm{\delta}) 
+ \frac{1}{2} \|\bm{\Omega}\bm{\beta}- \bm{\delta}\|_2^2
+ \frac{\lambda}{\rho} \sum_{o=1}^{|E|}w_o \| \bm{\delta}_o \|_1    
= \frac{1}{2}  \| \bm{\delta} - (\bm{\Omega}\bm{\beta}
 + \rho^{-1} \bm{\gamma} ) \|_2^2
+  \frac{\lambda}{\rho}  \sum_{o=1}^{|E|}w_o \| \bm{\delta}_o \|_1.
\label{obj_delta2}
\end{align}
Then, we divide the right-hand side of Eq. (\ref{obj_delta2}) into differentiable term and other term.
\begin{align*}
 \begin{cases}
\mathcal{F}(\bm{\delta}) =  \frac{1}{2}  \| \bm{\delta} - (\bm{\Omega}\bm{\beta}+ \rho^{-1} \bm{\gamma} ) \|_2^2 \\
\mathcal{G}(\bm{\delta}) =  \frac{\lambda}{\rho} \sum_{o=1}^{|E|}w_o \| \bm{\delta}_o \|_1
\end{cases}
\end{align*}
Using this, the proximity operator of $\mathcal{G}(\bm{\delta})$ is defined as follows:
\begin{align}
\mathrm{prox}_{\psi_o \| \cdot\|_1}(\bm{\nu}_o) 
= \argmin_{\bm{\delta}_o}(\psi_o \| \bm{\delta}_o \|_1 
+  \frac{1}{2} \| \bm{\delta}_o - \bm{\nu}_o \|_2^2 
).
\label{delta_proc}
\end{align}
Eq. (\ref{delta_proc}) is separable with respect to each component. Therefore, the problem reduces to solving, for each component $\xi$:
\begin{align*}
\argmin_{\delta_\xi} \psi |\delta_\xi| + \frac{1}{2}(\delta_\xi - v_\xi)^2.
\end{align*}
Therefore, the minimizer is obtained as Eq. ( \ref{delta_upd}).

\subsection{Update \texorpdfstring{$\bm{\gamma}$}{gamma}}
\label{subsubsec:update_gamma}

The update of $\bm{\gamma}$ is performed by the deviation from the constraint, using the gradient descent method as
\begin{align}
\bm{\gamma}^{(t+1)}
= \bm{\gamma}^{(t)} + \rho (\bm{\Omega}\bm{\beta}^{(t+1)}- \bm{\delta}^{(t+1)}). 
\label{gamma_upd}
\end{align}

The detail of the Algorithm of GrAUC-PFL is shown in $\mathrm{Algorithm}\ 1$.

{\footnotesize
\begin{spacing}{1.0}
\begin{algorithm}[H]
    \caption{Geographically regularized AUC-maximizing personalized federated learning (GrAUC-PFL)}
    \small
\footnotesize
    \label{alg1}
    \begin{algorithmic}[1]    
    \Require $\bm{X}_z  (z=1, 2, \cdots, Z), \bm{y}_z, \bm{\Omega}, \bm{H}, r, w_{uv}, \rho, \tau, q$ 
    \Ensure $\bm {\beta}_z \ (z=1, 2, \dots, Z), \bm {\delta}, \bm {\gamma}$
    \State Set $t \leftarrow 1$
    \State Set initial values ${\bm \beta_z^{(0)}}$, $\bm{\delta}^{(0)}$ and $\bm{\gamma}^{(0)}$
    \While{ $\mathcal{L}_\rho^{(t)} - \mathcal{L}_\rho^{(t+1)}$ $\geq \epsilon$ }
    \State {\bf Update $\bm{B}$}:
    \Statex {\bf Server}:
    \State Update $\bm{\beta}^{(t,g)} \leftarrow r^{-1}\bm{H}\bm{\beta}^{(t)} - r^{-1} \tau \big[ - \rho \bm{\Omega}^\top \bm{\delta}^{(t)} + \bm{\Omega}^\top \bm{\gamma^{(t)}}\big]$
    \Statex {\bf Local institution}:
    \Statex \quad Download $\bm{\beta}_z^{(t)}$ from Server
    \State Update $\bm{\beta}_z^{(t+1)} \leftarrow\bm{\beta}_z^{(t,g)} - r^{-1} \tau \nabla f_z (\bm{\beta}_z^{(t)}), \ $ where $\nabla f_z (\bm{\beta}_z)=\frac{1}{n_z m_z} \sum_{i=1}^{n_z} \sum_{j=1}^{m_z} \big(- \bm{d}_{ijz} \sigma(- (\bm{\beta}_z^\top \bm{d}_{ijz}+q)) \big)$
    \Statex \quad Upload $\bm{\beta}_z^{(t+1)}$ to Server
    \Statex {\bf Server}:
    \Statex \quad Integrate $\bm{\beta}_z^{(t+1)}$ to form $\bm{\beta}^{(t+1)}$ 
    \State Update $\bm{\delta}^{(t+1)}$based on Eq. (\ref{delta_upd})
    \State Update $\bm{\gamma}^{(t+1)} \leftarrow \bm{\gamma}^{(t)} + \rho (\bm{\Omega}\bm{\beta}^{(t+1)}- \bm{\delta}^{(t+1)})$ 
    \EndWhile
    \end{algorithmic}
\end{algorithm}
\label{alg}
\end{spacing}
}

\subsection{Selection of tuning parameters}
\label{subsec:CV}

We select the tuning parameters $\lambda$, $\rho$, and $w_{dist}$ for the weight of the regularization term $w_{uv}$, using stratified $K$-fold cross-validation. Stratified sampling is employed so that the proportion of positive and negative cases is approximately maintained in each fold. Patients are randomly assigned to folds with approximately equal fold sizes. At each step, GrAUC-PFL is trained jointly across all institutions using the corresponding training subsets. 
The candidates of the tuning parameters are sent to the server and used in server-side calculations. Once the AUC has been calculated at each institution, the results are aggregated at the central server to calculate the overall mean. 
The combination of parameters that maximized the mean validation AUC across institutions is selected. The combination of tuning parameters that achieved the highest mean validation AUC across the $K$-folds is selected. 
As for the calculations of the mean AUC values for CV, the following methods are available: 
\begin{align*}
\mathrm{AUC_{CV}}
= 
\frac{1}{K}
\sum_{k =1}^{K}
\frac{1}{n_z m_z}
\sum_{z=1}^Z
\sum_{(i,j)\in C_k }
I\!\left(
\hat{\bm{\beta}}_{zk}^\top (\bm{x}^+_{iz} - \bm{x}^-_{jz}) > 0
\right)
\end{align*}
where $C_k$ denotes the Cartesian product of the set of positive samples and that of negative samples belonging to $k$th fold. $\hat{\bm{\beta}}_{zk}$ represents estimated $\hat{\bm{\beta}}_{z}$ corresponding to $k$th fold.

\section{Numerical simulation}\label{sec:simulation}

\subsection{Simulation design}
\label{subsec:sim_design}

We conduct four numerical simulations to evaluate the performance of GrAUC-PFL. In this section, we describe the simulation design based on \cite{chen2025auc}. The settings for each of the four simulations are as follows: Simulation 1 examines the case where a relationship exists between the structure of the mean vector and the geographical group structure. Simulation 2 involves a scenario where groups exist within the mean structure that are related to geographical coordinates, but where the sample size varies by institutions. In simulation 3, the institutions have a mean structure in group structure, but this is not related to their geographical coordinates. Simulation 4 examines all institutions sharing a single mean structure, and no correspondence exists between the geographical and mean structures. 

In all simulations, positive and negative labels are denoted by $y$, and data that differ for each class are generated. Then, we generate the coordinates based on the proximity of institutions. 

The common settings is explained first. Let the number of institutions be $Z=20 \ $ ($z=1,2, \dots, 20$). The outcome is defined as $y_{hz} = \{0, 1\} \ (h=1, \dots, N_z)$, where $y_{hz} = 1$ represents a positive sample and $y_{hz} = 0$ represents a negative one. For each institution, $n_z$ and $m_z$ denote the number of positive and negative samples, respectively, such that $N_z = n_z + m_z$. The prevalence rate was set to be the same across all institutions, with $\pi \ (= n_z/N_z) \in \{ 0.3, 0.25, 0.2, 0.15, 0.1, 0.05 \}$. 

Next, set the dimension of the explanatory variables at institution $z$ as $\bm{X}_z \in \mathbb{R}^{N_z\times p}$ to $p=18$, designating $3$ of these variables as active variables and the remaining $15$ as noise. Let $\bm{x}_{hz} \in \mathbb{R}^p$ denote the vector of explanatory variables for individual $h$ at institution $z$. 
For negative samples ($y_{hz}=0$), we assume that $\bm{x}_{jz}^- \sim N(\bm{0},\bm{I}_p)$, and for positive samples ($y_{hz}=1$), 
\begin{align}
\bm{x}_{iz}^+ \sim N(\bm{\mu}^*_{z}, \bm{I}_p)
\label{x_pos}
\end{align}
where $\bm{\mu}^*_{z} = (\mu^*_{z}, \mu^*_{z}, \mu^*_{z}, 0, \dots, 0)^\top \in \mathbb{R}^p$. The first three variables are informative variables, and the remaining are noise. The settings for $\bm{\mu}_z^*$ are explained in each simulation setting.

Now, we explain the setting of GrAUC-PFL. The weight for each edge $(u,v)$ is calculated based on Eq. (\ref{w_weight}). We use R package {\choosefont{pcr}geosphere} \citep{R_geosph} to compute pairwise distance from the coordinate data, and the MST is obtained using {\choosefont{pcr}igraph} package \citep{R_igraph}.  The constant value $q = 0.01$, and each simulation was repeated $100$ times. For the initial values of the parameters, the coefficient vector $\bm{\beta}$ was initialized by randomly sampling from a standard normal distribution for each institution $z$; $\bm{\beta}_z^{(0)} \sim N(\bm{0}, \bm{I}_p)$. 
The auxiliary variable $\bm{\delta}$ is initialized as $\bm{\delta}^{(0)} = \bm{\Omega} \bm{\beta}^{(0)}$ to ensure that the consistency condition is satisfied at initializationzw. 
The dual variable $\bm{\gamma}$ is initialized as the zero vector. 

For the compared methods, we adopted two approaches. The first one is a local approach that maximizes the AUC surrogate independently at each institution, which is named 'individual'.
The second method is an approach that directly maximizes the AUC within the framework of Federated Averaging (FedAvg) \citep{FedAvg}, which estimates a common coefficient vector using conventional FL. 

The evaluation metric was the AUC for each institution by using test data and calculating the mean value across institutions. 
For the AUC calculation for each institution, we used {\choosefont{pcr}roc} function of the {\choosefont{pcr}pROC} package in R \citep{R_pROC}.

\noindent
{\bf Simulation 1}

Simulation 1 examines cases where similarities exist in the structure of institutions based on geographical proximity. The number of samples at each institution $z$ in the training data is set to $N_{z, train} \in \{50, 300 \}$, and the number of samples in the test data is set to $N_{z,test} = 100$. The $20$ institutions are divided into two clusters ($z=1$ to $10$ and $z=11$ to $20$) consisting of geographically adjacent institutions. 
Assuming the center coordinates of each cluster are $\bm{c}_1=(0.1, 0.1)^\top$ and $\bm{c}_2=(0.9, 0.3)^\top$ respectively, the coordinates of the individual institutions are generated as follows:
\begin{align*}
(\mathrm{lon}_z, \mathrm{lat}_z) = \bm{c}_{g(z)} + \bm{\epsilon}_z, \quad \bm{\epsilon}_z \sim N(0, 0.03^2I_2).
\end{align*}
where $g(z) =1$ for $z \leq 10$ and $g(z) =2$ for $z \geq 11$. Here, $g(z)$ is a function allocating clustering number for institution $z$.

Next, for the mean of the covariates for the positive samples, we set $\bm{\mu}_z^*$ as $\bm{\mu}^*_{lg(z)} = (\mu^*_{lg(z)}, \mu^*_{lg(z)}, \mu^*_{lg(z)}, 0, \dots, 0)^\top \in \mathbb{R}^p$ depending on $l$ and $g(z)$. The signal is controlled by the mean shift $\mu^*_{lg(z)}$ in the first three variables, while the remaining variables act as noise. $\mu^*_{lg(z)}$ is set for each group $g(z) \in \{1,2\}$, and three different levels $l=\{1,2,3\}$. The pattern for $\mu^*_{lg(z)}$ is shown in Table \ref{mu_x}.
Based on the above settings, training and test datasets are independently generated from the same data-generating process with the same prevalence $\pi$ using Eq. (\ref{x_pos}). 
\begin{table}[H]
\centering
\caption{The pattern of $\mu^*_{lg(z)}$. $l$ represents the levels of the mean structure and $g(z)$ is the group of institutions with geographical similarity.}\label{mu_x}
\begin{tabular}{c|cc}
\hline
& \multicolumn{2}{c}{group $g(z)$} \\
\hline
level $l$ & $1$ & $2$ \\
\hline
$1$ & $0.9$ & $-0.9$ \\
$2$ & $0.7$ & $-0.7$ \\
$3$ & $0.5$ & $-0.5$ \\
\hline
\end{tabular}
\end{table}
The tuning parameters $\lambda$ and $\rho$, and $w_{dist}$ are determined using $5$-fold cross-validation applied to the entire PFL model. \\

\noindent
{\bf Simulation 2}

Simulation 2 examines the situations where the number of patients in training data varies by institutions. The mean structure and the geographical coordinates are identical to those in simulation 1. 
To reflect heterogeneous sample sizes across institutions, the training sample size for each institution $N_{z, train}$ is randomly assigned from $\{30,50,100,200,300\}$ with equal probability and is fixed throughout each simulation replicate. The test sample size is set as $N_{z, test}=200$. As the minimum sample size of the training data is small, a 3-fold cross-validation is performed to generate positive/negative pairs for each fold. Furthermore, the sample sizes for positive and negative cases are calculated principally based on prevalence; however, for institutions with small sample sizes, the numbers are adjusted to ensure that at least four cases are allocated to each category. Other settings are same as those in simulation 1.

\noindent
{\bf Simulation 3}

As with simulation 1, the data in simulation 3 possesses a group structure based on the mean values of the covariates for positive samples; however, geographical coordinates are assumed independent of this mean-value structure.

Now we explain the settings. For each institution $(z=1,2,\dots,20)$, its longitude and latitude are independently generated as
\begin{align}
\mathrm{lon}_z \sim \operatorname{Unif}(130, 145), \quad
\mathrm{lat}_z \sim \operatorname{Unif}(31, 45).
\label{coords_sim3}
\end{align}
Thus, institutional locations are distributed randomly over a predefined rectangular geographical region, without imposing geographic group structure. The other settings including the mean structure of the covariates are identical to those in simulation 1.

\noindent
{\bf Simulation 4}

We examine cases where the mean structure is not related to geographical coordinates, for which institutions are generated in the same manner as Eq. (\ref{coords_sim3}) in simulation 3. To introduce continuous heterogeneity across institutions without imposing a group structure, institution-specific mean parameters are generated for each institution. Specifically, the mean vector for the positive class in the institution $z$ is then defined as
\begin{align*}
\bm{\mu}_z^*=(\mu^*_z,\mu^*_z,\mu^*_z,0,\ldots,0)^\top \in \mathbb{R}^p, \ \mu^*_z \sim \operatorname{Unif}(0.5, 0.9).
\end{align*}
Using this, the covariates for positive samples are generated based on Eq. (\ref{x_pos}).
All other simulation settings, including the number of samples for traning and test data, the generation of predictor variables and evaluation procedure were identical to those used in Simulation 1.

\subsection{Simulation results}\label{subsec:sim_results}

\begin{figure}[H] 
\begin{center}
\includegraphics[scale=0.47]{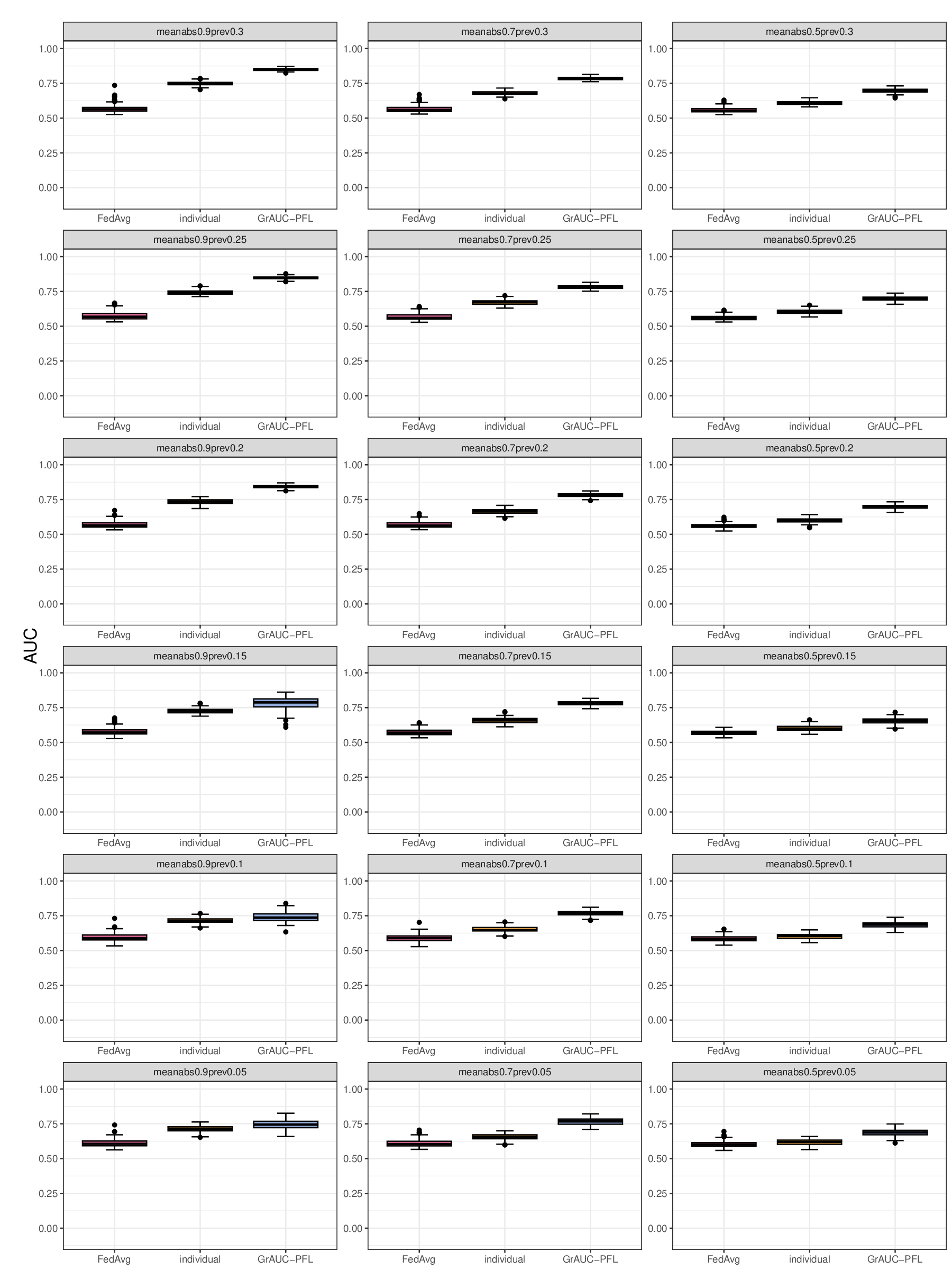}
\caption{The results of simulation 1 in $N_{z,train}=50$. The vertical axis shows the mean AUC as assessed by the institution, whilst the horizontal axis represents the method. 
For each plot, 'meanabs' mentions the mean structure of $\mu^*_{lg(z)}$ at Table \ref{mu_x} and 'prev' denotes the prevalence rate. 
The methods are as follows: 'GrAUC-PFL' refers to the proposed method, 'FedAvg' refers to the Federated Average, and 'individual”' refers to the method in which models are built for each institution and their AUC are evaluated separately.}
\label{sim_n50}
\end{center}
\end{figure}

\begin{figure}[H]
\begin{center}
\includegraphics[scale=0.47]{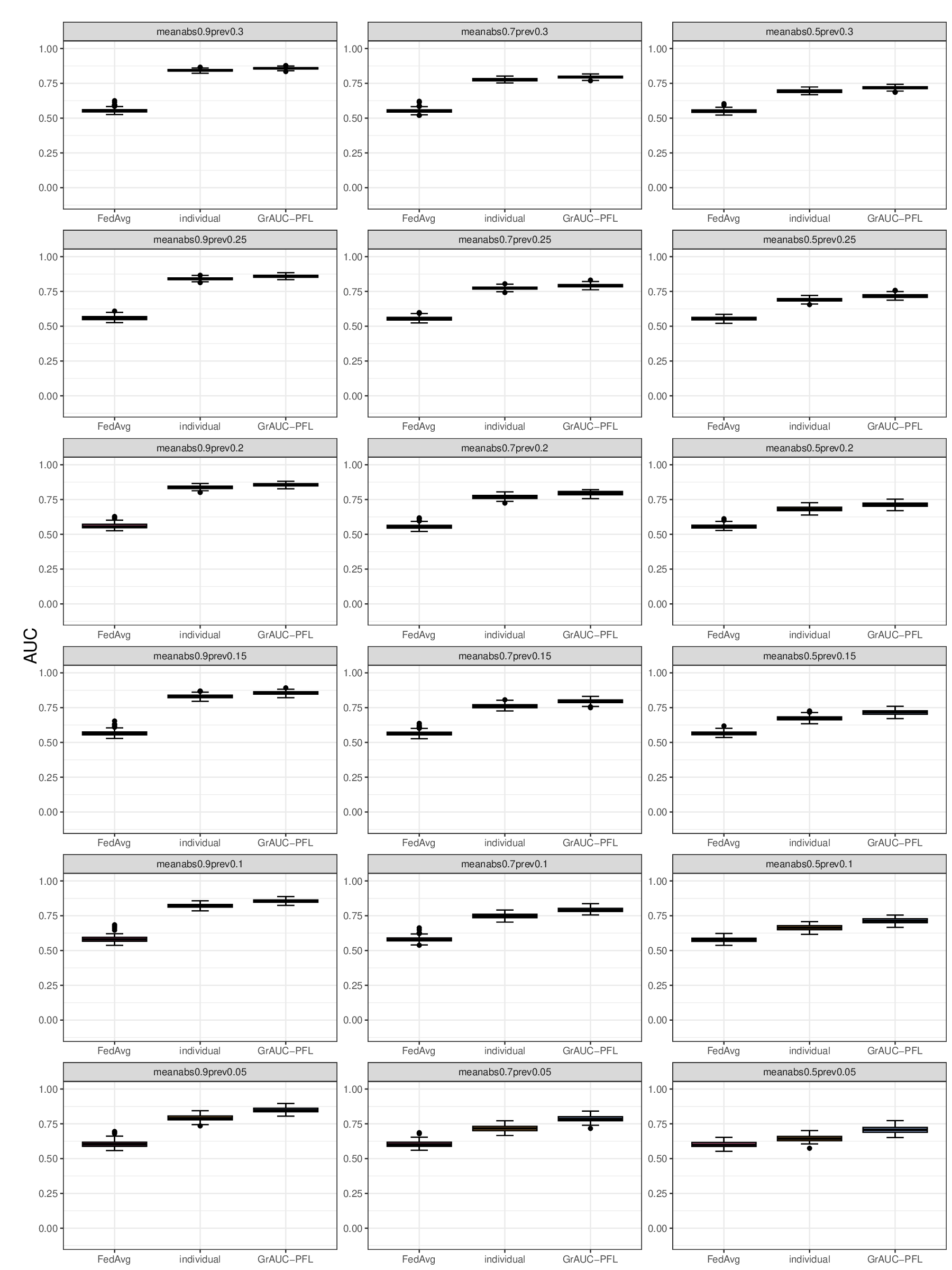}
\caption{The results of simulation 1 in $N_{z,train}=300$. }
\label{sim_n300}
\end{center}
\end{figure}

The results of the numerical simulation 1 are shown in Figure \ref{sim_n50} for $N_{z,train}=50$ and Figure \ref{sim_n300} for $N_{z,train}=300$. GrAUC-PFL achieved the highest mean AUC across all patterns for both $N_{z, train}=50$ and $300$. Its performance improved as the difference in $\mu^*_{lg(z)}$ increased, whereas FedAvg showed no notable change. Prevalence had relatively little effect on performance. For $N_{z, train}=300$, compared with the results for $N_{z, train}=50$, the AUC of the individual method was close to that of GrAUC-PFL, especially when the prevalence value was large. 
In simulation 2, GrAUC-PFL achieved better results in almost all patterns, shown in Figure \ref{sim2} in Appendices \ref{secC}. Similar to the results of simulation 1, the higher the signal and the higher prevalence, the greater the estimation accuracy. By contrast, FedAvg consistently exhibited the lowest performance. 
In simulation 3, shown  in Figure \ref{sim3_50} for $N_{z,train}=50$ and Figure \ref{sim3_300} for $N_{z,train}=300$ in Appendices \ref{secC}, the results of GrAUC-PFL were slightly better than those of the individual AUC-maximization method in $\bm{\mu}^*_{lg(z)}$ were higher. In the other scenarios, the individual method yielded good results, although the difference in values compared to GrAUC-PFL was not particularly large.  
In simulation 4, whose results are described in Figure \ref{sim4} in Appendices \ref{secC}, the results of GrAUC-PFL were almost the same as those of FedAvg when the prevalence was $0.3$ and $0.25$ for $N_{z,train}=50$. Although FedAvg was superior to GrAUC-PFL for the remaining patterns, the difference between them was slight.

\section{Real data application} \label{sec:realdata}

To evaluate the practical usefulness of GrAUC-PFL, we applied it to a real-world dataset. In this section, we first describe the characteristic of the dataset and preprocessing procedures, and then present the predictive performance compared to other methods and estimated parameter patterns obtained by the proposed.

\subsection{Data description and Experimental Setting} \label{subsec:data}

To evaluate the practical usefulness of GrAUC-PFL, we apply it to a real-world dataset. First, we explain the dataset for the application. The real-data application is conducted using the publicly available COVID-19 Case Surveillance Public Use Data with Geography, which contains nationwide surveillance records collected in the United States provided by the Centers for Disease Control and Prevention (CDC) \citep{cdc_covid_geo}. This dataset is collected from the beginning of $2020$ and is accessed on $2$ February $2026$. This study uses publicly available de-identified data; therefore, institutional review board approval and informed consent are not required. Owing to computational constraints, a subset of $100,000$ records from the original dataset is retrieved for the analysis. Records containing missing, unknown, or unusable variables values are excluded, resulting in $11,628$ complete cases for analysis, and the outcome is hospitalization. Six covariates are selected after excluding records with missing or unknown values. The details of these variables are listed in Table \ref{covariates}. 

These data include state and county names and the corresponding geographic codes. Counties are treated as individual units, and geographical proximity is considered. To calculate pairwise geographic distances, the geographic coordinates for each county are derived from county boundary shapefiles obtained from the U.S. Census Bureau's Topologically Integrated Geographic Encoding and Referencing (TIGER)/Line shapefiles using the R package {\choosefont{pcr}tigris} \citep{tigris}. To calculate the pairwise distance, we use the R package {\choosefont{pcr}geosphere} and the {\choosefont{pcr}igraph} to construct a graph of the MST. The analysis data comprise $18$ counties, arranged as shown in Figure \ref{us_map}. There are some neighboring counties on the west and east coasts. Table \ref{data_county} lists the patient data by county. Prevalence was calculated as the proportion of hospitalization in each county. The number of hospitalizations in some counties is small and the prevalence in many counties is low. 

Next, we explain the procedure used. For each county, we construct a training dataset by randomly sampling hospitalized and non-hospitalized patients separately, and use the remaining observations as the test dataset. First, we partition the data within each county into training and test datasets using outcome-stratified sampling. Approximately $20$\% of samples from each outcome class are allocated to the training dataset, whereas the remaining samples are assigned to the test dataset. For counties with relatively small numbers of inpatients, the allocation is adjusted to ensure that both outcome classes are represented in the training and test datasets. The tuning parameter $\lambda$, $\rho$, and $w_{dist}$ are determined via stratified $3$-fold cross-validation using the training dataset. 

The evaluation index is the mean AUC of each county. To calculate the AUC for each county, the R package {\choosefont{pcr}pROC} is used. The compared methods are FedAvg and a localized approach that individually maximize alternative measures of AUC in each county, similar to the numerical simulation.
\begin{table}[h]
\centering
\caption{Data summary by county.}
\label{covariates}
\begin{tabular}{c|l}
\hline
Variable & Categories\\
\hline
age\_group & 0--17, 18--49, 50--64, $\geq$65 yrs\\
sex & Female, Male\\
race & American Indian, Alaska Native, Asian,
Black, \\
& Multiple/Other, Native Hawaiian, Other Pacific Islander, White\\
ethnicity & Hispanic, Non-Hispanic\\
current\_status & Laboratory-confirmed case, Probable case
\\
symptom\_status & Asymptomatic, Symptomatic \\
\hline
\end{tabular}
\end{table}
%
\begin{figure}[H]
\begin{center}
\includegraphics[scale=0.7]{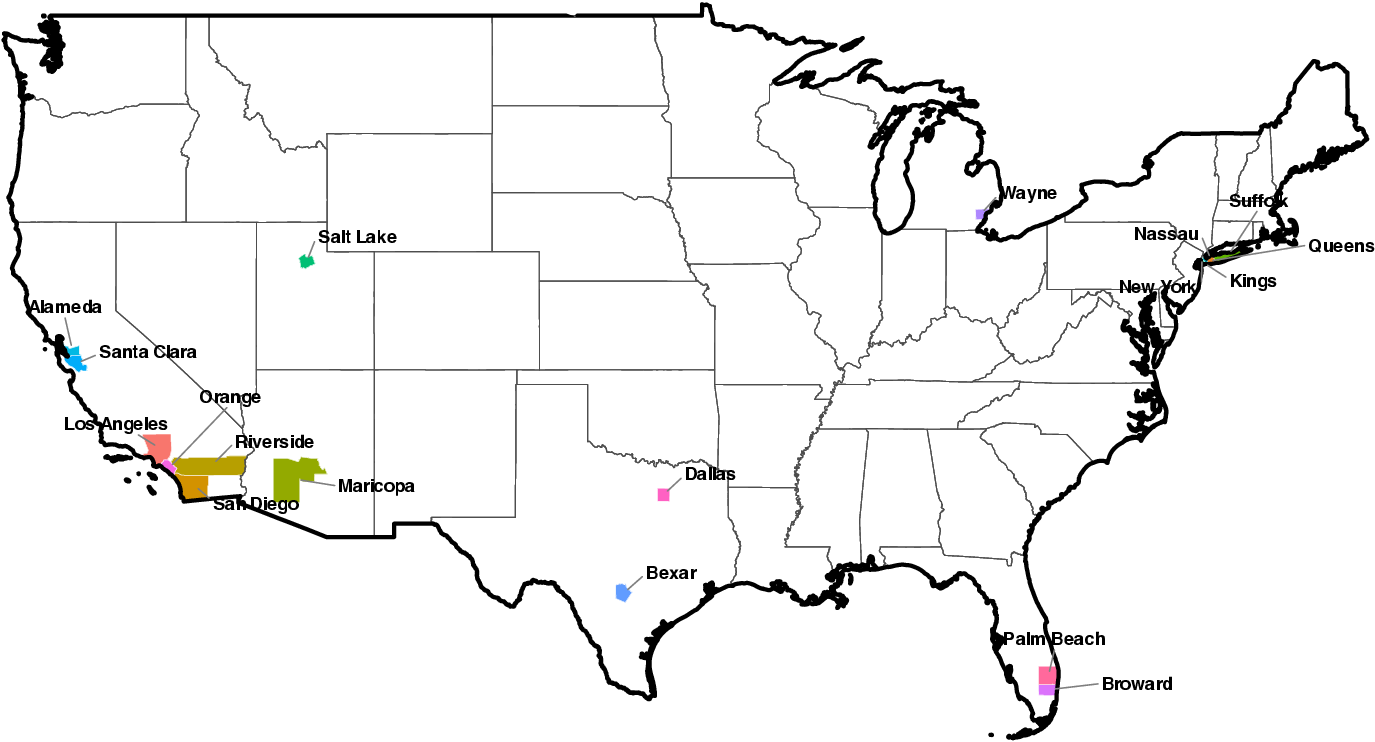}
\caption{Geographic distribution across the contiguous United States, where Alaska, Hawaii, and U.S. territories are excluded from the map visualization.}
\label{us_map} 
\end{center}
\end{figure}

\begin{table}[h]
\centering
\caption{Data distribution by county.}
\scalebox{0.8}{
\begin{tabular}{ccrrr|ccrrr}
\hline
State & County & Total $N_z$ & Hospitalization  & Prevalence & State &County & Total $N_z$ & Hospitalization & Prevalence\\
\hline
UT & Salt Lake & $1,935$ & $33$ & $0.017$ & NY & Kings &  $212$ &  $189$ &$0.892$ \\
FL & Palm Beack & $888$ & $32$ & $0.036$ & NY & Queen & $131$ & $111$ & $0.847$\\
FL & Broward & $1,282$ & $37$ & $0.029$ & TX & Dallas & $195$ & $20$ & $0.103$\\
TX & Bexar & $849$ & $37$ & $0.044$ & NY & Nassau & $166$ & $4$ & $0.024$\\
AZ & Maricopa & $1,107$ & $256$ & $0.231$ & NY & Suffolk & $1,212$ & $39$ & $0.032$\\
CA & Riverside & $522$ & $14$ & $0.027$ & CA & Santa Clara & $274$ & $9$ & $0.033$\\
MI & Wayne & $603$ & $42$ & $0.070$ & CA & Alameda & $524$ & $79$ & $0.151$\\
CA & Los Angels & $1,006$ & $22$ & $0.022$ & NY & New York & $34$ & $27$ & $0.794$\\
CA & San Diego & $578$ & $57$ & $0.099$ & CA & Orange & $110$ & $10$ & $0.091$\\
\hline
\label{data_county}
\end{tabular}
}
\end{table}

\subsection{Results} \label{subsec:real_result}

The results of the mean AUC of each county were shown in Table \ref{result_realdata}. GrAUC-PFL was superior to the compared methods. Table \ref{result_county} describes the AUC by county. GrAUC-PFL showed better results in half of the counties. For Los Angeles, Santa Clara, and Riverside, the AUC of the individual AUC maximisation method was lower, whereas the GrAUC-PFL and FedAvg methods were superior for those counties. However, the results of FedAvg in Dallas, Orange, Nassau and New York were rather better than those of the other methods. 
Next, Figure \ref{beta_heatmap} shows the heatmap of $\hat{\bm{\beta}}$ of GrAUC-PFL. The dendrogram confirmed that the data had been successfully clustered into the following groups: East Coast (Nassau, Suffolk, New York, Kings, and Queens), West Coast (Riverside, Maricopa, San Diego, Los Angeles, and Orange), and Florida (Broward and Palm Beach), along with adjacent regions within the same state. 
\begin{table}[h]
\centering
\caption{Results of mean AUC applying the real-world data.}
\begin{tabular}{c|ccc}
\hline
 & GrAUC-PFL & FedAvg & individual \\
\hline
AUC & $\bm{0.651}$ & $0.628$ & $0.633$ \\
\hline
\end{tabular}
\label{result_realdata}
\end{table}
\begin{table}[h]
\centering
\caption{Results of AUC by each county.}
\scalebox{0.8}{
\begin{tabular}{crrr|crrr}
\hline
County & GrAUC-PFL & FedAvg & individual & County & GrAUC-PFL & FedAvg & individual\\
\hline
Salt Lake & $0.659$ & $0.409$ & $0.689$ & Kings & $0.792$ & $0.781$ & $0.865$\\
Palm Beach & $0.469$ & $0.493$ & $0.452$ & Queens & $0.644$ & $0.439$ & $0.651$ \\
Broward & $0.573$ & $0.434$ & $0.573$ & Dallas & $0.752$ & $0.837$ & $0.776$ \\
Bexar & $0.703$ & $0.544$ & $0.694$ & Nassau & $0.269$ & $0.438$ & $0.269$ \\
Maricopa & $0.719$ & $0.711$ & $0.725$ & Suffolk & $0.742$ & $0.742$ &$0.742$ \\
Riverside & $0.759$ & $0.758$ & $0.441$ & Santa Clara & $0.594$ & $0.520$ & $0.397$\\
Wayne & $0.814$ & $0.777$ & $0.815$ & Alameda & $0.773$ & $0.721$ & $0.775$ \\
Los Angels & $0.581$ & $0.488$ & $0.485$ &New York & $0.786$ & $0.950$ & $0.595$ \\
San Diego & $0.627$ & $0.388$ & $0.626$ & Orange & $0.455$ & $0.870$ & $0.818$\\
\hline
\end{tabular}
}
\label{result_county}
\end{table}
%
\begin{figure}[H]
\begin{center}
\includegraphics[scale=0.8]{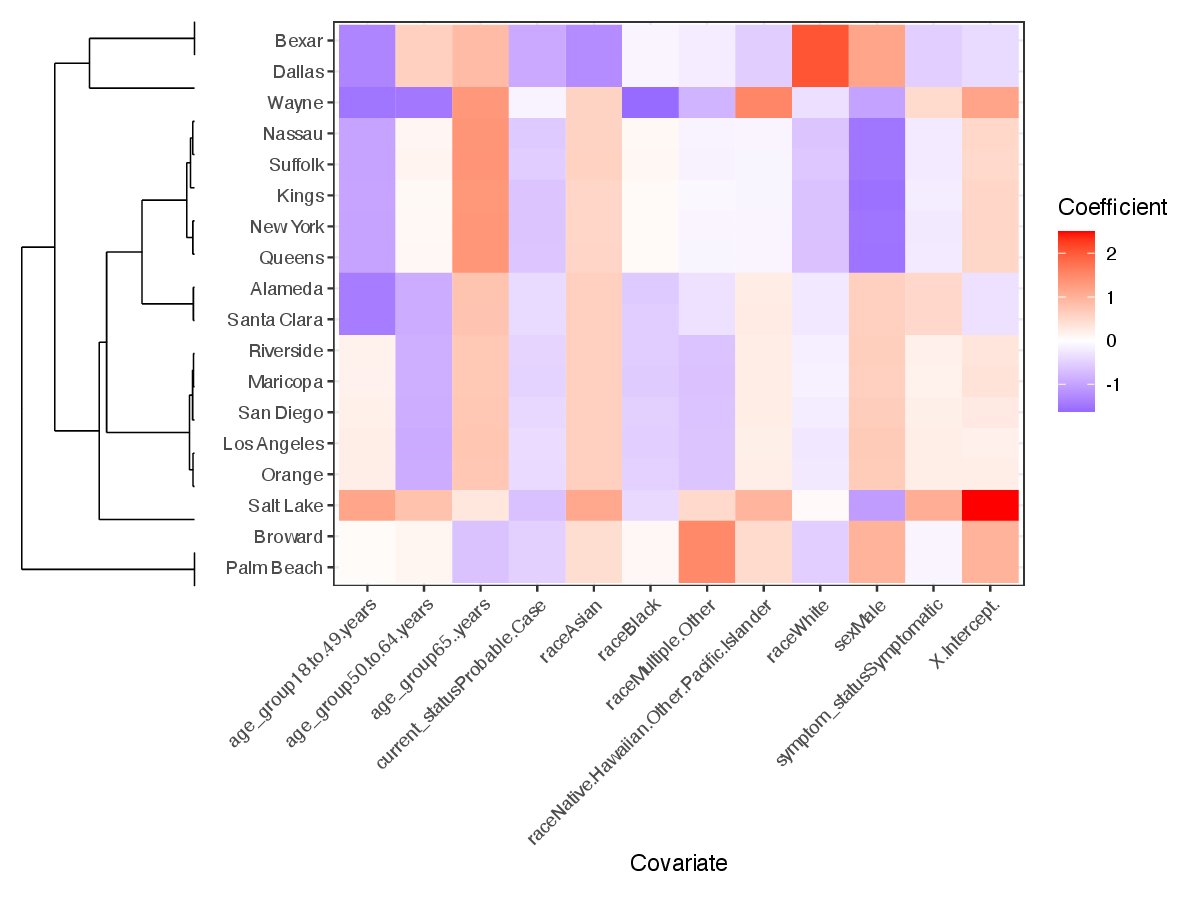}
\caption{Heatmap of $\hat{\bm{\beta}}$ of GrAUC-PFL.}
\label{beta_heatmap} 
\end{center}
\end{figure}

\section{Discussion}\label{sec:discussion}

In numerical simulations, we examined the performance in four setting scenarios compared to other AUC maximization methods. In simulations 1 and 2, where geographically neighboring institutions shared similar underlying data distributions, GrAUC-PFL demonstrated superior performance compared to the other methods. Simulation 2 further demonstrated that this advantage was maintained, even when training sample sizes varied across institutions, reflecting common features of multicenter studies. In contrast, the individual AUC maximization yielded lower AUC values, suggesting that relying solely on local data may be insufficient when sample size are limited. FedAvg also showed limited performance even when $N_{z,\mathrm{train}}=300$, indicating that a single global model may fail to capture institutional heterogeneity. These findings suggest that GrAUC-PFL can benefit institutions with limited local data by utilizing information from geographically similar institutions while retaining institution-specific model. Simulations 3 and 4 further examined scenarios in which geographical proximity did not reflect similarities in the data distribution. Under these conditions, GrAUC-PFL performed comparable to individual method, without significant performance degradation. These findings suggest that the benefit of spatial regularization depends on whether geographical proximity reflects data similarity, while its misspecification did not result in substantial performance loss in our simulations.

The real-world data application to COVID-19 data further demonstrated potential benefit of the GrAUC-PFL, which achieved the highest mean AUC across counties. Across the counties, GrAUC-PFL outperformed the compared methods in Los Angeles, Santa Clara, and Riverside, where disease prevalence was relatively low. These results suggest that borrowing information from geographically related counties can complement limited local information. In counties where individual models performed best, GrAUC-PFL generally achieved comparable AUC values, suggesting that geographic information sharing did not compromise local predictive performance. However, FedAvg outperformed GrAUC-PFL in some counties such as Nassau and Orange. The coefficient heatmap showed that GrAUC-PFL estimated similar coefficient patterns among geographically neighboring counties, consistent with the intended effect of geographical regularization. This suggests that, in these counties, the overall trends common to all counties may have provided more useful information than local adaptations based on geographical proximity. 

\section{Conclusion}\label{sec:conclusion}

In this study, we proposed GrAUC-PFL, a personalized federated learning framework that maximizes AUC directly, while incorporating geographic similarity among institutions. GrAUC-PFL enables building personalized models, while leveraging information from geographically related institutions without disclosing data externally. Numerical simulations and its application to real-world data demonstrated its effectiveness in improving discriminative performance.

For limitation of GrAUC-PFL is that its assumption that geographically proximate institutions tend to share similar model parameters, which may not always hold in practice. Future work could extend the framework to alternative measure of similarity, such as similarities in local models or data distributions \citep[e.g.][]{wang2023, FedAMP}. In addition, although we adopted a logistic-based surrogate loss for AUC maximization, investigating alternative surrogate losses and their effects on predictive performance and computational efficiency remains an important direction for future research.

\bibliographystyle{unsrtnat}
\bibliography{references}  

\clearpage

\appendix
\appendixpage
\section{Proof of Lemma 1}\label{secA1}

\begin{proof} 
The first term of the logistic surrogate function at institution $z$ is 
\begin{align}
f_z(\bm{\beta}_z)
= 
\frac{1}{n_z m_z} \sum_{i=1}^{n_z} \sum_{j=1}^{m_z} 
\log \big(1+\exp[-(\bm{\beta}_z^\top \bm{x}_{iz}^{+} - \bm{\beta}_z^\top \bm{x}_{jz}^{-} )-q] \big)
\label{log_sur}
\end{align}
Here, we denote $\bm{d}_{ijz} = \bm{x}_{iz}^+ - \bm{x}_{jz}^-$. 
Eq. (\ref{log_sur}) at a pair can be expressed as 
\begin{align}
\ell_{ijz}(\bm{\beta}_z)
= 
\log \big(1+\exp[-(\bm{\beta}_z^\top \bm{d}_{ijz})-q] \big)
\label{log_sur_pr}
\end{align}
The gradient and Hessian matrix of Eq. (\ref{log_sur_pr}) are given as follows:
\begin{align*}
\nabla \ell_{ijz}(\bm{\beta}_z)
&= 
-
\sigma(-\bm{\beta}_z^\top \bm{d}_{ijz}-q)\bm{d}_{ijz},  \\
\nabla^2 \ell_{ijz}(\bm{\beta}_z)
&= \sigma(-\bm{\beta}_z^\top \bm{d}_{ijz}-q)\{1 - \sigma(-\bm{\beta}_z^\top \bm{d}_{ijz}-q) \}
\bm{d}_{ijz} \bm{d}_{ijz} ^\top.
\end{align*}
where $\sigma(\cdot)$ is a sigmoid function.
Let $u_{ijz}=-\bm{\beta}_z^\top \bm{d}_{ijz} - q$.
Since $\sigma(u_{ijz}) \in (0,1)$, it follows that
\begin{align*}
\sigma(u_{ijz})\{1-\sigma(u_{ijz})\} \leq \frac{1}{4},
\end{align*}
Then, since $\bm{d}_{ijz}\bm{d}_{ijz}^\top$ is positive semidefinite, for each pair,
\begin{align*}
\sigma(u_{ijz})\{1-\sigma(u_{ijz})\} \bm{d}_{ijz}\bm{d}_{ijz}^\top
\preceq
\frac{1}{4}\bm{d}_{ijz}\bm{d}_{ijz}^\top.
\end{align*}
where $\preceq$ denotes the positive semidefinite ordering.
Averaging over all pairs yields
\begin{align}
\nabla^2 f_z(\bm{\beta}_z)
&= \frac{1}{n_zm_z} \sum_{i=1}^{n_z}
\sum_{j=1}^{m_z}
\sigma(u_{ijz})\{1-\sigma(u_{ijz})\} \bm{d}_{ijz}\bm{d}_{ijz}^\top \nonumber\\
& \preceq \frac{1}{n_zm_z} \sum_{i=1}^{n_z}
\sum_{j=1}^{m_z} \frac{1}{4}\bm{d}_{ijz}\bm{d}_{ijz}^\top
 \nonumber\\
&=
\frac{1}{4n_zm_z}
\sum_{i=1}^{n_z}
\sum_{j=1}^{m_z}
\bm{d}_{ijz}\bm{d}_{ijz}^{\top}
\label{los_lips}
\end{align}
Here, we set
\begin{align*}
\mu_z
=
\varsigma_{\max} \bigg(
\frac{1}{4n_zm_z}
\sum_{i=1}^{n_z}
\sum_{j=1}^{m_z}
\bm{d}_{ijz}\bm{d}_{ijz}^{\top} \bigg), 
\end{align*}
where $\varsigma_{max}(\bm{O})$ denotes the largest eigenvalue of $\bm{O}$. 
Hence, the Hessian is uniformly bounded. 
Let $\bm{Q}_z = \frac{1}{4n_zm_z}\sum_{i=1}^{n_z}
\sum_{j=1}^{m_z}
\bm{d}_{ijz}\bm{d}_{ijz}^{\top}$,  for any $\bm{g} \in \mathbb{R}^p$, it holds
\begin{align*}
\bm{g}^\top \nabla^2 f_z(\bm{\beta}_z) \bm{g} 
\leq
\bm{g}^\top \bm{Q}_z \bm{g}
\leq 
\varsigma_{\max} (\bm{Q}_z) \|\bm{g} \|_2^2,
\end{align*}
where $\| \cdot \|_2$ is Euclidean norm.
Then, 
\begin{align*}
\bm{g}^\top \nabla^2 f_z(\bm{\beta}_z) \bm{g} 
&\leq
\varsigma_{\max} (\bm{Q}_z)\| \bm{g} \|_2\\
&=
 \bm{g}^\top\big(\varsigma_{\max} (\bm{Q}_z)\bm{I}\big) \bm{g}\\
\nabla^2 f_z(\bm{\beta}_z)
&\preceq 
\varsigma_{\max} (\bm{Q}_z) \bm{I}
\end{align*}
From Eq. (\ref{los_lips}), $\varsigma_{\max} (\bm{Q}_z) \bm{I} - \nabla^2 f_z(\bm{\beta}_z)$ is positive semidefinite. Therefore, for all eigenvalue of $\nabla^2 f_z(\bm{\beta}_z)$, it holds
\begin{align*}
\varsigma_k (\nabla^2 f_z(\bm{\beta}_z))
&\leq 
\varsigma_{\max} (\bm{Q}_z)
\end{align*}
%
For the largest eigenvalue, it also holds:
\begin{align}
\varsigma_{\max} (\nabla^2 f_z(\bm{\beta}_z))
\leq 
\varsigma_{\max} (\bm{Q}_z).
\label{eigenmax_ineq}
\end{align}
Here, $\nabla^2 f_z(\bm{\beta}_z)$ is an symmetric positive-semidefinite matrix, 
\begin{align*}
\varsigma_{\max} (\nabla^2 f_z(\bm{\beta}_z))
= \| \nabla^2 f_z(\bm{\beta}_z)\|_2. 
\end{align*}
Then, Eq. (\ref{eigenmax_ineq}) can be expressed as 
\begin{align}
\| \nabla^2 f_z(\bm{\beta}_z)\|_2
\leq 
\varsigma_{\max} (\bm{Q}_z) = \mu_z.
\label{mu_z_hessian}
\end{align}
Next,  for any $\bm{\beta}_{z}, \bm{\beta}_z^\prime \in \mathbb{R}^p$, by the mean-value theorem,
\begin{align*}
\nabla f_z(\bm{\beta}_{z}) - \nabla f_z(\bm{\beta}_{z}^\prime)
=
\int_0^1 \nabla^2 f_z
\big(
\bm{\beta}_{z}^\prime + c(\bm{\beta}_{z}-\bm{\beta}_{z}^\prime)
\big)
(\bm{\beta}_{z}-\bm{\beta}_{z}^\prime)dc
\end{align*}
Take the norm 
\begin{align}
\|\nabla f_z(\bm{\beta}_{z})-\nabla f_z(\bm{\beta}_{z}^\prime)\|_2
&\leq 
\int_0^1 
\|\nabla^2 
f_z(\bm{\beta}_{z}^\prime+c(\bm{\beta}_{z}-\bm{\beta}_{z}^\prime))\|_2
\|\bm{\beta}_{z}-\bm{\beta}_{z}^\prime\|_2dc \nonumber\\
\Longrightarrow \|\nabla f_z(\bm{\beta}_{z})-\nabla f_z(\bm{\beta}_{z}^\prime)\|_2
& \leq 
\bigg(
\int_0^1 
\|\nabla^2 
f_z(\bm{\beta}_{z}^\prime +c(\bm{\beta}_{z}-\bm{\beta}_{z}^\prime))\|_2
dc \bigg)
\|\bm{\beta}_z-\bm{\beta}_z^\prime\|_2
\label{norm_mean_theo}
\end{align}
From Eq. (\ref{mu_z_hessian}), 
\begin{align*}
\int_0^1 
\|\nabla^2 
f_z(\bm{\beta}_{z}^\prime+c(\bm{\beta}_{z}-\bm{\beta}_{z}^\prime))\|_2
dc
\leq
\int_0^1 \mu_z dc = \mu_z,
\end{align*}
Then, Eq. (\ref{norm_mean_theo}) can be expressed as
\begin{align}
\|\nabla f_z(\bm{\beta}_{z})-\nabla f_z(\bm{\beta}_{z}^\prime)\|_2
\leq 
\mu_z
\|\bm{\beta}_z-\bm{\beta}_z^\prime\|_2
\label{lips_cont_z}
\end{align}
Therefore, the gradient at institution $z$ is $\mu_z$ Lipschitz continuous. 
For $\bm{\beta}$, 
\begin{align*}
f(\bm{\beta}) = \frac{1}{Z}\sum_{z=1}^Z f_z(\bm{\beta}_z).
\end{align*}
Then, 
\begin{align*}
\nabla f(\bm{\beta}) = \frac{1}{Z}
\nabla f_z(\bm{\beta}_z), \ (z=1, 2,\cdots, Z)
\end{align*}
For $\bm{\beta}, \bm{\beta}^\prime$,
\begin{align*}
\|\nabla f(\bm{\beta})-\nabla f(\bm{\beta}^\prime)\|_2^2
=
\frac{1}{Z^2}
\sum_{z=1}^Z
\|\nabla f(\bm{\beta}_z) - \nabla f(\bm{\beta}_z^\prime)\|_2^2
\end{align*}
%
Therefore, with Eq. (\ref{lips_cont_z}), we have, 
\begin{align*}
\|\nabla f(\bm{\beta})-\nabla f(\bm{\beta}^\prime)\|_2^2
&=
\frac{1}{Z^2}
\sum_{z=1}^Z
\|\nabla f(\bm{\beta}_z) - \nabla f(\bm{\beta}_z^\prime)\|_2^2\\
&\leq 
\frac{1}{Z^2}
\sum_{z=1}^Z
\mu_z^2
\|\bm{\beta}_z-\bm{\beta}_z^\prime\|_2^2\\
&\leq\frac{\mu^2_{\max}}{Z^2}
\sum_{z=1}^Z
\|\bm{\beta}_z-\bm{\beta}_z^\prime\|_2^2\\
&=
\frac{\mu^2_{\max}}{Z^2}\|\bm{\beta}-\bm{\beta}^\prime\|_2^2.
\end{align*}
where $\mu_{\max} = \max_{1\leq z \leq Z} \mu_z$. 
From above, 
\begin{align*}
\|\nabla f(\bm{\beta})-\nabla f(\bm{\beta}^\prime)\|_2
\leq
\frac{\mu_{\max}}{Z}\|\bm{\beta}-\bm{\beta}^\prime\|_2
\leq \mu_{\max}\|\bm{\beta}-\bm{\beta}^\prime\|_2.
\end{align*}
Thus, we have $L=\max_z(\mu_z)$.
\end{proof}

\newpage

\section{Proof of Theorem 2}\label{secB}

\begin{proof} 
We prove that the following Eq. (\ref{perfl_lemma1}) is the majorizing function of GrAUC-PFL to update $\bm{\beta}$, and demonstrate that Theorem 2 holds when $r> \rho \tau \varsigma_{\max}(\bm{\Omega}^\top\bm{\Omega})+\max( \frac{\tau\mu}{2}, 1)$, which is

\begin{align*}
\mathcal{L}_\rho (\bm{\beta}^{(t+1)}, \bm{\delta}^{(t)}, \bm{\gamma}^{(t)})
- \mathcal{L}_\rho (\bm{\beta}^{(t)}, \bm{\delta}^{(t)}, \bm{\gamma}^{(t)}) 
\leq 
- \bigg[
\frac{\varsigma_{\min}(\bm{H})}{\tau}
+ \frac{\rho \cdot \varsigma_{\min}(\bm{\Omega}^\top \bm{\Omega})}{2}
- \frac{\mu}{2}
\bigg]
\| \bm{\beta}^{(t+1)}- \bm{\beta}^{(t)}\|_2^2.
\end{align*}
First of all, differentiate Eq. (\ref{obj_beta_mm1}) with respect to $\bm{\beta}$ and set the result equal to zero.
\begin{align}
\nabla \tilde{\mathcal{L}}_\rho(\bm{\beta}; \bm{\beta}^{(t)}, \bm{\delta}^{(t)} ,\bm{\gamma}^{(t)} )
=
\nabla f(\bm{\beta}^{(t)}) 
+ \frac{1}{ \tau}{\bm{H}}(\bm{\beta} - \bm{\beta}^{(t)})
+ \bm{\Omega}^\top\bm{\gamma}^{(t)}
+ \rho \bm{\Omega}^\top(\bm{\Omega}\bm{\beta}- \bm{\delta}^{(t)})=\bm{0}.
\label{lag_def}
\end{align}
Substitute $\bm{\beta}^{(t+1)}$ into $\bm{\beta}$ in Eq. (\ref{lag_def})
\begin{align}
\nabla f(\bm{\beta}^{(t)}) 
- \frac{1}{\tau}{\bm{H}}(\bm{\beta}^{(t)} - \bm{\beta}^{(t+1)})
+ \bm{\Omega}^\top\bm{\gamma}^{(t)}
+ \rho \bm{\Omega}^\top(\bm{\Omega}\bm{\beta}^{(t+1)}- \bm{\delta}^{(t)})=\bm{0}.
\label{lag_def2}
\end{align}
Apply $(\bm{\beta}^{(t)} - \bm{\beta}^{(t+1)})^\top$ to Eq. (\ref{lag_def2}) from the left:
\begin{align}
0
=
(\bm{\beta}^{(t)} - \bm{\beta}^{(t+1)})^\top
\bigg[
\nabla f(\bm{\beta}^{(t)}) 
- \frac{1}{\tau}{\bm{H}}(\bm{\beta}^{(t)} - \bm{\beta}^{(t+1)})
+ \bm{\Omega}^\top\bm{\gamma}^{(t)}
+ \rho \bm{\Omega}^\top(\bm{\Omega}\bm{\beta}^{(t+1)}- \bm{\delta}^{(t)})
\bigg].
\label{lag_def3}
\end{align}
From Lemma 1, if the smooth function $f:\mathbb{R}^p \mapsto \mathbb{R}$ has Lipschitz continuous gradient, the following inequality holds: 
\begin{align}
f(\bm{\beta}^{(t+1)}) 
&\leq 
f(\bm{\beta}^{(t)})+ \nabla f(\bm{\beta}^{(t)})^\top (\bm{\beta}^{(t+1)} - \bm{\beta}^{(t)}) + \frac{\mu}{2}\|\bm{\beta}^{(t+1)} - \bm{\beta}^{(t)} \|^2_2 \nonumber\\
f(\bm{\beta}^{(t)}) - f(\bm{\beta}^{(t+1)}) 
& \geq 
- \nabla f(\bm{\beta}^{(t)})^\top (\bm{\beta}^{(t+1)} - \bm{\beta}^{(t)}) - \frac{\mu}{2}\|\bm{\beta}^{(t+1)} - \bm{\beta}^{(t)} \|^2_2 \nonumber\\
\nabla f(\bm{\beta}^{(t)})^\top (\bm{\beta}^{(t)} - \bm{\beta}^{(t+1)})
& \leq 
f(\bm{\beta}^{(t)}) - f(\bm{\beta}^{(t+1)})
+\frac{\mu}{2}\|\bm{\beta}^{(t+1)} - \bm{\beta}^{(t)} \|^2_2 
\label{desc_lemma}
\end{align}
With Eq. (\ref{desc_lemma}), Eq. (\ref{lag_def3}) can be derived the following inequality:
\begin{align}
\bm{0}
&=
(\bm{\beta}^{(t)} - \bm{\beta}^{(t+1)})^\top
\bigg[
\nabla f(\bm{\beta}^{(t)}) 
- \frac{1}{\tau}{\bm{H}}(\bm{\beta}^{(t)} - \bm{\beta}^{(t+1)})
+ \bm{\Omega}^\top\bm{\gamma}^{(t)}
+ \rho \bm{\Omega}^\top(\bm{\Omega}\bm{\beta}^{(t+1)}- \bm{\delta}^{(t)})
\bigg] \nonumber\\
&\leq
f(\bm{\beta}^{(t)}) - f(\bm{\beta}^{(t+1)})
+\frac{\mu}{2}\|\bm{\beta}^{(t)} - \bm{\beta}^{(t+1)} \|^2_2 
- \frac{1}{\tau} \| \bm{\beta}^{(t)} - \bm{\beta}^{(t+1)}\|^2_{\bm{H}} 
\nonumber \\
& \qquad \qquad 
+ (\bm{\beta}^{(t)} - \bm{\beta}^{(t+1)})^\top \bm{\Omega}^\top \bm{\gamma}^{(t)}
+ \rho (\bm{\beta}^{(t)} - \bm{\beta}^{(t+1)})^\top \bm{\Omega}^\top(\bm{\Omega}\bm{\beta}^{(t+1)} - \bm{\delta}^{(t)}) \nonumber\\
&=
f(\bm{\beta}^{(t)}) - f(\bm{\beta}^{(t+1)})
+\frac{\mu}{2}\|\bm{\beta}^{(t)} - \bm{\beta}^{(t+1)} \|^2_2 
- \frac{1}{\tau} \| \bm{\beta}^{(t)} - \bm{\beta}^{(t+1)}\|^2_{\bm{H}} \nonumber \\
& \qquad \qquad 
+ \bm{\gamma}^{(t)\top}(\bm{\Omega}\bm{\beta}^{(t)} -  \bm{\Omega}\bm{\beta}^{(t+1)})^\top
+ \rho (\bm{\Omega}\bm{\beta}^{(t)} - \bm{\Omega}\bm{\beta}^{(t+1)})^\top (\bm{\Omega}\bm{\beta}^{(t+1)} - \bm{\delta}^{(t)})
\label{lag_def_all}
\end{align}
where$\|\bm{a}\|_{\bm{H}}= \bm{a}^\top\bm{a}$.
Next, as for the fifth term of Eq. (\ref{lag_def_all}),
\begin{align}
&\bm{\gamma}^{(t)} (\bm{\Omega}\bm{\beta}^{(t)} - \bm{\Omega} \bm{\beta}^{(t+1)}) 
+ \bm{\gamma}^{(t)\top} \bm{\delta}^{(t)} - \bm{\gamma}^{(t)\top} \bm{\delta}^{(t)} \nonumber\\
=& \bm{\gamma}^{(t)} (\bm{\Omega}\bm{\beta}^{(t)}- \bm{\delta}^{(t)}) - \bm{\gamma}^{(t)} (\bm{\Omega}\bm{\beta}^{(t+1)}- \bm{\delta}^{(t)})
\label{gamma_term}
\end{align}
Regarding the sixth term of Eq. (\ref{lag_def_all}),
we use the following vector identity in the same manner in \cite{perFL-RSR}:
\begin{align*}
(\bm{a}-\bm{b})^\top(\bm{b}-\bm{c})
=\frac{1}{2}(\|\bm{a} -\bm{c}\|^2_2 - \|\bm{a} -\bm{b}\|^2_2 - \|\bm{b} -\bm{c}\|^2_2).
\end{align*}
Then, the sixth term of Eq. (\ref{lag_def_all}) can be rewritten as
\begin{align}
&\rho (\bm{\Omega}\bm{\beta}^{(t)} - \bm{\Omega} \bm{\beta}^{(t+1)})^\top (\bm{\Omega}\bm{\beta}^{(t+1)}- \bm{\delta}^{(t)}) \nonumber\\
=&\frac{\rho}{2}
\big( 
\| \bm{\Omega}\bm{\beta}^{(t)} - \bm{\delta}^{(t)}\|^2_2
- \| \bm{\Omega}\bm{\beta}^{(t)} - \bm{\Omega} \bm{\beta}^{(t+1)}\|^2_2
- \| \bm{\Omega}\bm{\beta}^{(t+1)} - \bm{\Omega} \bm{\delta}^{(t)}\|^2_2
\big).
\label{delta_term}
\end{align}
With Eq. (\ref{gamma_term}) and Eq. (\ref{delta_term}), Eq. (\ref{lag_def_all}) can be expressed as follows:
\begin{align}
\bm{0} 
\leq
& f(\bm{\beta}^{(t)})
- f(\bm{\beta}^{(t+1)})
+ \frac{\mu}{2} \| \bm{\beta}^{(t)} - \bm{\beta}^{(t+1)} \|^2_2
- \frac{1}{\tau}  \| \bm{\beta}^{(t)} - \bm{\beta}^{(t+1)}\|^2_{\bm{H}} \nonumber\\
&+\bm{\gamma}^{(t)} (\bm{\Omega}\bm{\beta}^{(t)}- \bm{\delta}^{(t)})
- \bm{\gamma}^{(t)} (\bm{\Omega}\bm{\beta}^{(t+1)}- \bm{\delta}^{(t)}) \nonumber\\
&+\frac{\rho}{2}\| \bm{\Omega}\bm{\beta}^{(t)} - \bm{\delta}^{(t)}\|^2_2
- \frac{\rho}{2}\| \bm{\Omega}\bm{\beta}^{(t)} - \bm{\Omega} \bm{\beta}^{(t+1)}\|^2_2
- \frac{\rho}{2}\| \bm{\Omega}\bm{\beta}^{(t+1)} - \bm{\delta}^{(t)}\|^2_2
\label{maj_beta2}
\end{align}
Here, the first, fifth, and seventh term of the right-hand side in Eq. (\ref{maj_beta2}) express the Lagrangian function $\mathcal{L}_\rho (\bm{\beta}^{(t)}, \bm{\delta}^{(t)}, \bm{\gamma}^{(t)}) $, and the second, sixth, and ninth term in Eq. (\ref{maj_beta2}) is the Lagrangian function $\mathcal{L}_\rho (\bm{\beta}^{(t+1)}, \bm{\delta}^{(t)}, \bm{\gamma}^{(t)})$, respectively. 
Therefore, 
\begin{align}
\bm{0} \leq
\mathcal{L}_\rho(\bm{\beta}^{(t)}, \bm{\delta}^{(t)}, \bm{\gamma}^{(t)})
-\mathcal{L}_\rho(\bm{\beta}^{(t+1)}, \bm{\delta}^{(t)}, \bm{\gamma}^{(t)})
+ \frac{\mu}{2} \| \bm{\beta}^{(t)} - \bm{\beta}^{(t+1)} \|^2_2
- \frac{1}{\tau}  \| \bm{\beta}^{(t)} - \bm{\beta}^{(t+1)}\|^2_{\bm{H}} 
- \frac{\rho}{2}\|\bm{\Omega}\bm{\beta}^{(t)} - \bm{\Omega}\bm{\beta}^{(t+1)} \|^2_2 \nonumber\\
\mathcal{L}_\rho(\bm{\beta}^{(t+1)}, \bm{\delta}^{(t)}, \bm{\gamma}^{(t)})
-\mathcal{L}_\rho(\bm{\beta}^{(t)}, \bm{\delta}^{(t)}, \bm{\gamma}^{(t)})  \leq
\frac{\mu}{2} \| \bm{\beta}^{(t+1)} - \bm{\beta}^{(t)} \|^2_2
- \frac{1}{\tau}  \| \bm{\beta}^{(t+1)} - \bm{\beta}^{(t)}\|^2_{\bm{H}} 
- \frac{\rho}{2}\|\bm{\beta}^{(t+1)} - \bm{\beta}^{(t)} \|_{\bm{\Omega}^\top\bm{\Omega}}
\label{maj_beta3}
\end{align}
To show that the expression (\ref{maj_beta3}) is monotonically decreasing, the right-hand side should be simplified to the form $-c \|\bm{\beta}^{(t+1)}-\bm{\beta}^{(t)} \|^2_2$. 
By the Rayleigh quotient theorem for a symmetric matrix $\bm{M}$,
\begin{align}
   \varsigma_{\min}(\bm{M})
   &\leq\frac{\bm{x}^{*\top} \bm{M}\bm{x}^*}{\|\bm{x}^*\|^2_2} \leq \varsigma_{\max}(\bm{M}), \quad \forall \bm{x}^* \neq \bm{0} \nonumber\\
   \varsigma_{\min}(\bm{M})\|\bm{x}^*\|^2_2
   &\leq\bm{x}^{*\top} \bm{M}\bm{x}^* \leq \varsigma_{\max}(\bm{M})\|\bm{x}^*\|^2_2
   \label{raylie}
\end{align}
Applying Eq. (\ref{raylie}) to the second and third terms of the right-hand side in Eq. (\ref{maj_beta3}),respectively. When $\Delta\coloneq \bm{\beta}^{(t+1)} - \bm{\beta}^{(t)}$,
\begin{align*}
\| \Delta \|^2_{\bm{H}} \geq \varsigma_{\min} (\bm{H})\|\Delta\|^2_2, \quad 
\| \Delta \|^2_{\bm{\Omega}^\top \bm{\Omega}} \geq \varsigma_{\min} (\bm{\Omega}^\top \bm{\Omega})\|\Delta\|^2_2.
\end{align*}
Substituting the corresponding terms from Eq.  (\ref{maj_beta3}) gives
\begin{align}
-\frac{1}{\tau}\| \Delta \|^2_{\bm{H}} &\leq -\frac{1}{\tau}\varsigma_{\min} (\bm{H})\|\Delta\|^2_2, \quad {\rm and }\label{rayl_h}\\
- \frac{\rho}{2} \| \Delta \|^2_{\bm{\Omega}^\top \bm{\Omega}} 
&\leq 
- \frac{\rho}{2} \varsigma_{\min} (\bm{\Omega}^\top \bm{\Omega})\|\Delta\|^2_2.
\label{rayl_omega}
\end{align}
Therefore, the right-hand side of Eq. (\ref{maj_beta3}) can be rewritten with Eq. (\ref{rayl_h}) and Eq. (\ref{rayl_omega}) as
\begin{align*}
\mathcal{L}_\rho(\bm{\beta}^{(t+1)}, \bm{\delta}^{(t)}, \bm{\gamma}^{(t)})
-\mathcal{L}_\rho(\bm{\beta}^{(t)}, \bm{\delta}^{(t)}, \bm{\gamma}^{(t)})  
\leq
- \bigg[
 \frac{\varsigma_{\min}(\bm{H})}{\tau}
 + \frac{\rho \varsigma_{\min}(\bm{\Omega}^\top \bm{\Omega})}{2}
 - \frac{\mu}{2}
 \bigg]
 \|\bm{\beta}^{(t+1)} - \bm{\beta}^{(t)} \|^2_2.
\end{align*}
Eq. (\ref{perfl_lemma1}) ensures that the augmented Lagrangian function decreases at each step of updating $\bm{\beta}$, since $\| \bm{\beta}^{(t+1)}- \bm{\beta}^{(t)}\|_2^2$ is strictly positive under the above condition. This can be achieved when  
\begin{align}
\frac{\varsigma_{\min}(\bm{H})}{\tau}
+ \frac{\rho \cdot \varsigma_{\min}(\bm{\Omega}^\top \bm{\Omega})}{2}
- \frac{\mu}{2}
>0.
\label{lemma1_const}
\end{align}
As mentioned, $\bm{H}$, defined in Eq. (\ref{H_set}), is required to be a positive definite so that the constructed function defined in Eq. (\ref{perfl_lemma1}) is a valid upper-bounding function.  
To investigate the conditions under which $\bm{H}$ is positive definite, we consider the eigen decomposition of $\bm{\Omega}^\top \bm{\Omega} = \bm{Q}\bm{\Lambda} \bm{Q}^\top$.
\begin{align*}
\bm{H} &= r \bm{I} - \rho \tau \bm{\Omega}^\top \bm{\Omega}\\
&= r \bm{I} - \rho \tau \bm{Q}\bm{\Lambda}\bm{Q}^\top\\
&= r\bm{Q}\bm{Q}^\top - \rho \tau \bm{Q}\bm{\Lambda}\bm{Q}^\top\\
&= \bm{Q} (r\bm{I} - \rho \tau \bm{\Lambda})\bm{Q}^\top
\end{align*}
Here, $\bm{Q}$ is an orthogonal matrix consisting of the eigenvectors of $\bm{\Omega}$ and $\bm{\Lambda}= \mathrm{diag}(\varsigma_1(\bm{\Omega}^\top \bm{\Omega}), \varsigma_2(\bm{\Omega}^\top \bm{\Omega}),\dots, \varsigma_{Zp}(\bm{\Omega}^\top \bm{\Omega}))$ is a diagonal matrix consisting of the corresponding eigenvalues. 
For this, $r\bm{I} - \rho \tau \bm{\Lambda}$ is a diagonal matrix, and its diagonal entries are given by $r\bm{I} - \rho \tau \bm{\Lambda} = \mathrm{diag}\big(r - \rho \tau \varsigma_1(\bm{\Omega}^\top \bm{\Omega}), \varsigma_2(\bm{\Omega}^\top \bm{\Omega}), \dots, r - \rho \tau \varsigma_{Zp} (\bm{\Omega}^\top \bm{\Omega})\big)$. 
$\bm{\Omega}^\top \bm{\Omega}$ is positive semidefinite, all its eigenvalues are nonnegative. 
Therefore, the eigenvalues of $\bm{H}$ are given by 
\begin{align}
\varsigma_{k^*}(\bm{H}) = r - \rho\tau \varsigma_{k^*}(\bm{\Omega^\top \Omega}).
\label{eigen_H}
\end{align}
The condition that $\bm{H}$ is positive definite is equivalent to the condition that $\varsigma_{k^*}(\bm{H})$ is greater than $0$ for some $k^*$. 
Therefore, $\bm{H}$ becomes a positive definite matrix by setting 
$r > \rho \tau \varsigma_{k^*}(\bm{\Omega^\top \Omega})$. In this case, $r$ represents the minimum value required for $\bm{H}$ to be a positive definite matrix.
This condition ensures the convexity of the surrogate function. From Eq. (\ref{eigen_H}), $r - \rho \tau \varsigma_{k^*}(\bm{\Omega}^\top \bm{\Omega})$ is decreasing in $\varsigma_{k^*}(\bm{\Omega}^\top \bm{\Omega})$. Therefore, $\varsigma_{\min}(\bm{H}) = r - \rho \tau \varsigma_{\max}(\bm{\Omega}^\top \bm{\Omega})$. Substituting $r - \rho \tau\varsigma_{\max}(\bm{\Omega}^\top \bm{\Omega})$ into Eq. (\ref{lemma1_const}):
\begin{align*}
\frac{1}{\tau}
\big(
r - \rho \tau \varsigma_{\max}(\bm{\Omega}^\top \bm{\Omega})
\big)
+ \frac{\rho}{2}
 \varsigma_{\min}(\bm{\Omega}^\top \bm{\Omega})
- \frac{\mu}{2}
>0\\
\frac{r}{\tau}
> \rho\varsigma_{\max}(\bm{\Omega}^\top \bm{\Omega})
- \frac{\rho}{2}\varsigma_{\min}(\bm{\Omega}^\top \bm{\Omega})
+ \frac{\mu}{2}\\
r 
> \rho \tau \varsigma_{\max}(\bm{\Omega}^\top \bm{\Omega})
- \frac{\rho \tau}{2}\varsigma_{\min}(\bm{\Omega}^\top \bm{\Omega})
+ \frac{\tau \mu}{2}
\end{align*}
Here, since $\varsigma_{\min}(\bm{\Omega}^\top \bm{\Omega})$, we have $- \frac{\rho}{2}\varsigma_{\min}(\bm{\Omega}^\top \bm{\Omega})\leq 0$. 
Therefore, a sufficient condition is given by 
$r 
> \rho \tau \varsigma_{\max}(\bm{\Omega}^\top \bm{\Omega}) + \frac{\tau \mu}{2}$. 
Furthermore, to ensure the positive definite of $\bm{H}$,  \cite{perFL-RSR} sets $r$ as 
\begin{align*}
r 
> \rho \tau \varsigma_{\max}(\bm{\Omega}^\top \bm{\Omega})
+ \max \bigg(\frac{\tau\mu }{2}, 1 \bigg).
\end{align*}

\end{proof}

\section{Results of numerical simulations}\label{secC}

\begin{figure}[H]
\begin{center}
\includegraphics[scale=0.48]{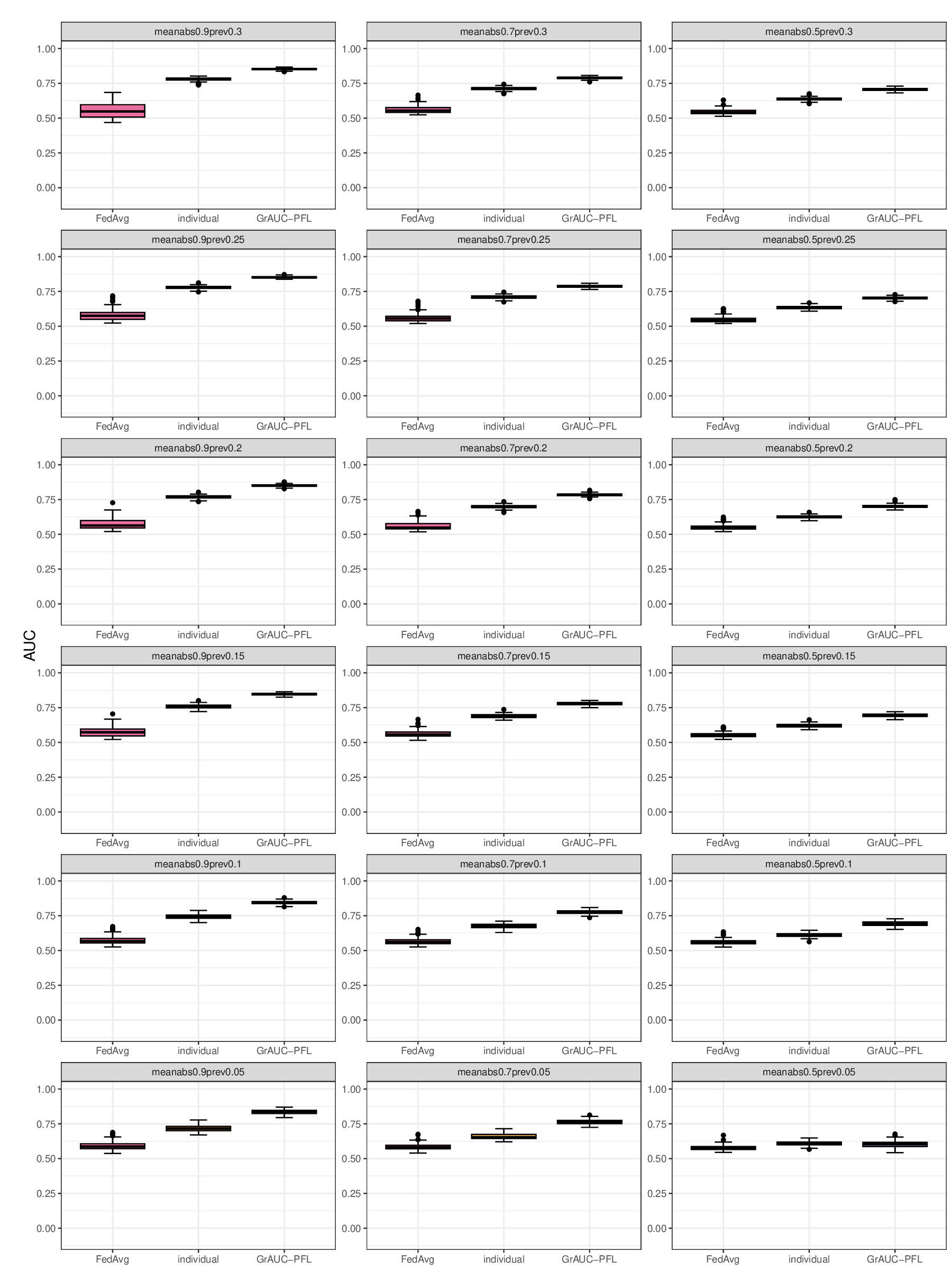}
\caption{The results of simulation 2 in that each institution owns different sample size from $N_{z,train}= \{30, 50, 100, 200, 300 \}$. }
\label{sim2}
\end{center}
\end{figure}
%

\begin{figure}[H]
\begin{center}
\includegraphics[scale=0.47]{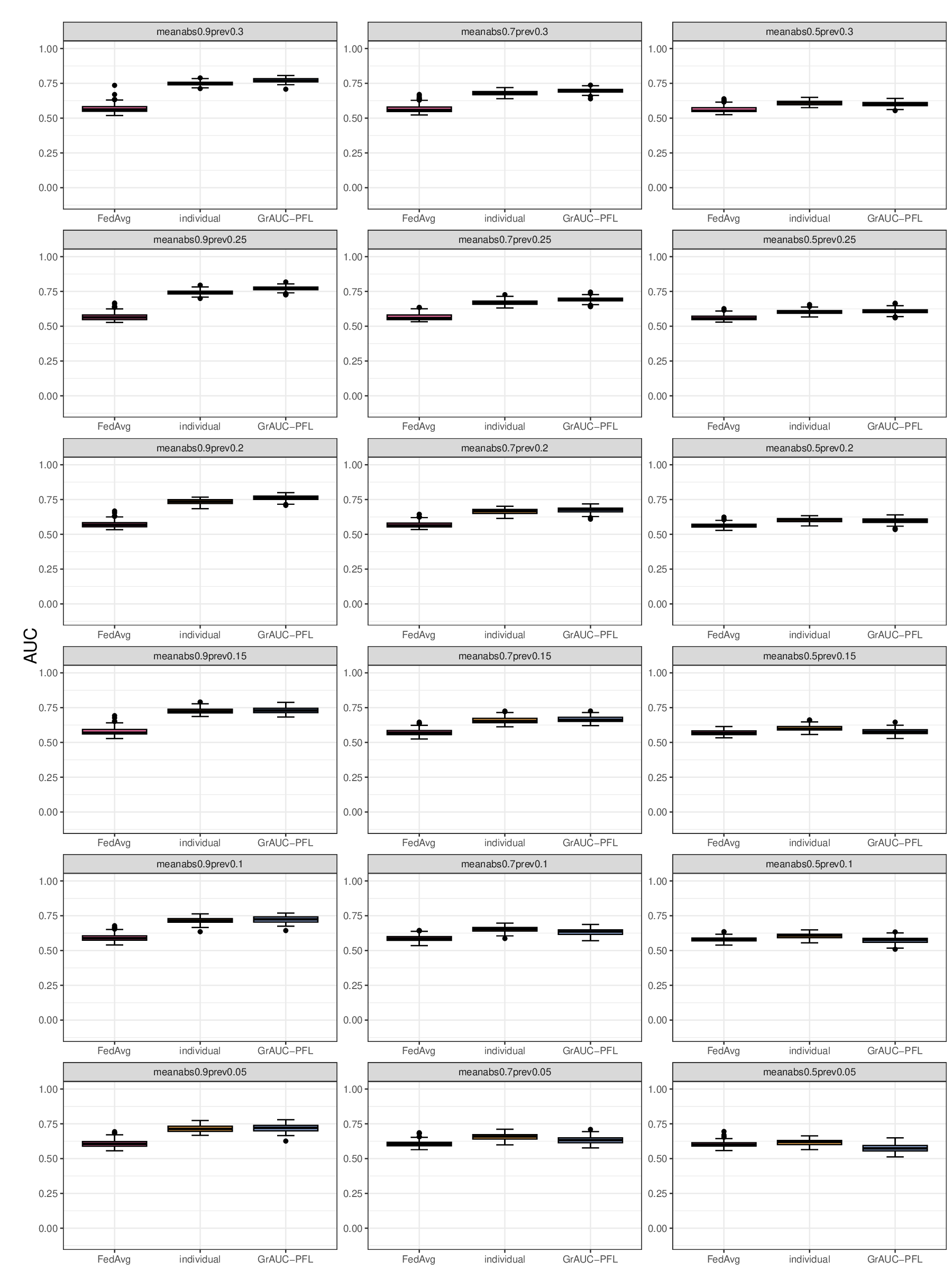}
\caption{The results of simulation 3 in $N_{z,train}=50$. }
\label{sim3_50}
\end{center}
\end{figure}
%

\begin{figure}[H]
\begin{center}
\includegraphics[scale=0.47]{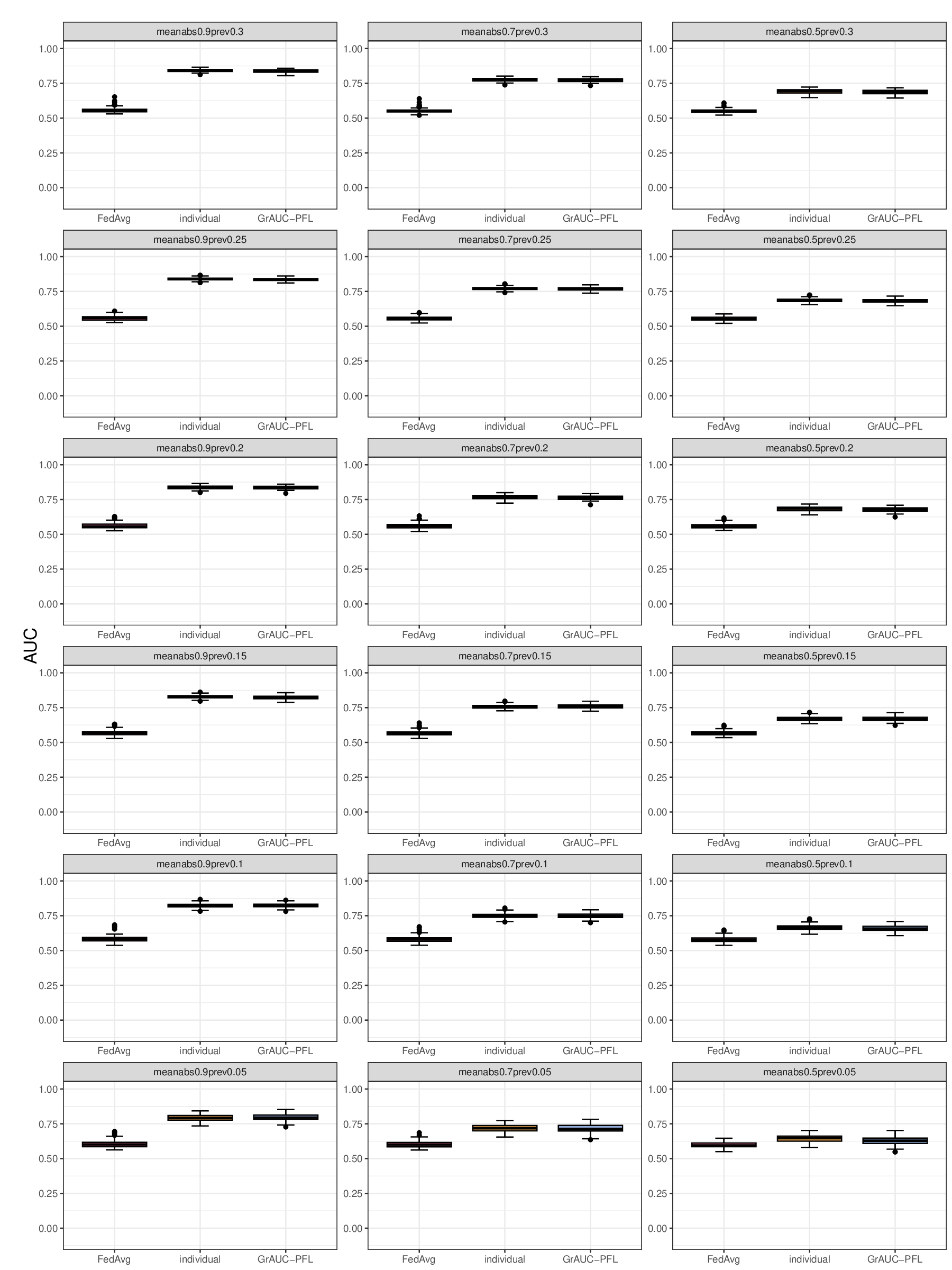}
\caption{The results of simulation 3 in $N_{z,train}=300$. }
\label{sim3_300}
\end{center}
\end{figure}
%

\begin{figure}[H]
\begin{center}
\includegraphics[scale=0.45]{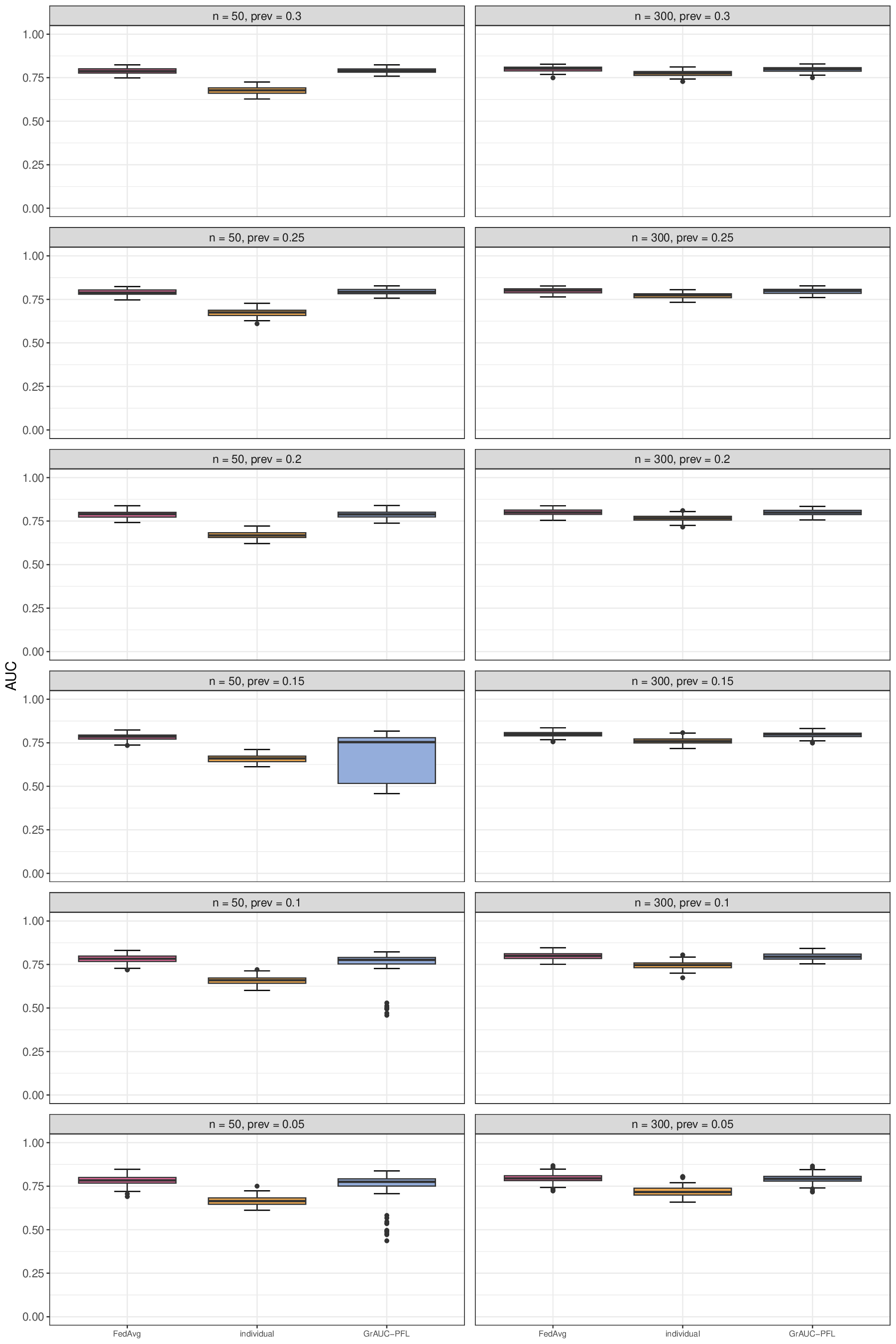}
\caption{The results of simulation 4. The plots in the left column are results at $N_{z,train}=50$ at each institution and those in the right column are at $N_{z,train}=300$. }
\label{sim4}
\end{center}
\end{figure}

\end{document}